\documentclass{amsart}

\usepackage[mathlines]{lineno}

\usepackage{afterpage}
\usepackage{enumitem}
\usepackage[margin=1in]{geometry}
\usepackage{graphicx}
\usepackage{wrapfig}
\usepackage{xcolor}

\usepackage{thmtools}
\declaretheorem[numberwithin=section]{lemma}
\declaretheorem{proposition, theorem, corollary, remark, definition, approximation}[
style=plain,
sibling=lemma
]

\usepackage{amsmath,amssymb,interval,mathtools}
\mathtoolsset{centercolon}

\newcommand\reals{\mathbb{R}}

\newcommand\CC{\mathbb{C}}
\newcommand\loss{\mathcal{L}}
\newcommand\normal{\mathcal{N}}

\DeclareMathOperator\E{\mathbb{E}}
\let\P\relax
\DeclareMathOperator\P{\mathbb{P}}
\DeclareMathOperator\var{Var}
\DeclareMathOperator\erf{erf}
\newcommand\mask{B}

\newcommand\W{\mathbf{W}}
\DeclareMathOperator\cov{cov}
\DeclareMathOperator\corr{corr}
\DeclareMathOperator\card{card}
\newcommand\bigO{\mathcal{O}}
\newcommand\smallO{
    \mathchoice
    {{\scriptstyle\mathcal{O}}}{{\scriptstyle\mathcal{O}}}{{\scriptscriptstyle\mathcal{O}}}{\scalebox{.1}{$\scriptscriptstyle\mathcal{O}$}}}
\DeclareMathOperator\Tr{Tr}

\usepackage[normalem]{ulem}

\usepackage[colorlinks,allcolors=blue]{hyperref}
\usepackage[capitalize,nameinlink]{cleveref}
\let\citep\cite
\labelcrefrangeformat{enumi}{#3#1#4--#5#2#6}
\crefname{approximation}{Approximation}{Approximations}

\title{On the Stability of the Jacobian Matrix in Deep Neural Networks}
\author[B.\ Dadoun]{Benjamin Dadoun}
\author[S.\ Hayou]{Soufiane Hayou}
\author[H.\ Salam]{Hanan Salam}
\author[M.\ E.\ A.\ Seddik]{Mohamed El Amine Seddik}
\author[P.\ Youssef]{Pierre Youssef}

\begin{document}

    \begin{abstract}
        Deep neural networks are known to suffer from exploding or vanishing gradients as depth increases, a phenomenon closely tied to the spectral behavior of the input-output Jacobian. Prior work has identified critical initialization schemes that ensure Jacobian stability, but these analyses are typically restricted to fully connected networks with i.i.d.\ weights. In this work, we go significantly beyond these limitations: we establish a general universality theorem for Jacobian products that accommodates sparsity (such as that introduced by pruning) and weakly correlated weights. Our results rely on recent advances in random matrix theory, and provide rigorous guarantees for spectral stability in a much broader class of network models. This extends the theoretical foundation for initialization schemes in modern neural networks with structured and dependent randomness.
    \end{abstract}

    \maketitle
    \thispagestyle{empty}

    \section{Introduction}

    Despite their impressive performance across a wide range of applications, training deep neural networks (DNNs) remains a challenging task,  requiring extensive hyperparameter tuning. A key factor that affects the trainability of DNNs is the behavior of the input-output Jacobian, which quantifies the sensitivity of the network's output to perturbations in its input~\citep{saxe2013exact}. Improper initialization can lead to vanishing or exploding Jacobian/gradients as network depth increases, potentially causing the training process to converge to suboptimal local optima or diverge altogether.

    Prior research on the Jacobian in DNNs has primarily examined  its behavior at initialization \citep{pennington2017, collins2023Asymptotic}, where  network weights are typically drawn from random distributions, such as  independent and identically distributed (i.i.d) Gaussians, or random orthogonal matrices~\cite{pennington2017, hanin2020products, chhaibi2022free}. These studies have demonstrated, both theoretically and empirically, a correlation between ``stable'' Jacobian behavior at initialization and improved training and generalization performance. Similar works have investigated   other related quantities at initialization such as input-output covariance kernel and gradient propagation dynamics, to derive the ``Edge of Chaos'' initialization regime \citep{poole2016exponential, samuel2017deep, hayou2019impact, hayou2022curse}, which guarantees gradient stability and improved information propagation through depth.

    However, these analyses are limited to the setting of random, uncorrelated initializations and do not readily extend to more structured or correlated weight distributions. Yet in many practical scenarios, such as data-dependent initialization~\cite{krahenbuhl2015data,chumachenko2022feedforward}, meta-learning~\cite{finn2017model}, or fine-tuning from pretrained weights~\cite{li2018measuring}, networks are initialized with non-i.i.d., often correlated weights~\cite{mishkin2015all}. Recent theoretical and empirical work has shown that even mild correlations in the weight matrices can significantly alter signal and gradient propagation dynamics, potentially improving convergence and generalization~\cite{jin2020does}. Still, the behavior of the Jacobian in the presence of correlated weights remains poorly understood, particularly in the large-depth regime where signal propagation is most sensitive. Understanding how weight correlations affect Jacobian stability is therefore crucial for explaining the training dynamics of deep networks in realistic settings.

    Similarly, another underexplored case is that of sparse networks. Pruning is a common technique that sparsifies neural networks to reduce their memory and computational footprint~\citep{lecun_pruning, hassibi_pruning, lee_snip2018, wangGRASP, franklelottery2019}. While score-based pruning methods such as SNIP~\citep{lee_snip2018}, GraSP~\citep{wangGRASP} and their variants~\citep{hayou2021robust}, aim to preserve important connectivity at initialization, the effect of pruning on the input-output Jacobian remains poorly understood. Existing works suggest that unstructured sparsity can disrupt gradient flow~\cite{evci2022gradient}, but lack theoretical guarantees on Jacobian stability. This gap motivates a closer analysis of Jacobian behavior in pruned networks and other non-standard settings.

    In this work, we take a step in the direction of understanding these regimes and focus on two scenarios:
    \begin{itemize}
        \item \emph{Correlated weights}: we investigate the Jacobian behavior in the fixed-depth infinite-width regime with correlated weights, and study the impact of depth on training stability in this case. Our goal here is to obtain an upperbound on weight correlation level that guarantees depth stability of the Jacobian.
        \item \emph{Sparse networks}: we study the Jacobian of pruned (sparse) networks, which can provide insights on how to adjust training algorithms to maintain stability after \emph{pruning}. We consider Random Pruning and  Magnitude-based pruning for tractability of the analysis. Our goal here is to show that different pruning algorithms require different treatments to preserve depth stability of the Jacobian.
    \end{itemize}

    We address these questions by developing a general universality theorem for Jacobian products (\cref{th:jacobian-universality}) based on recent breakthroughs in random matrix theory \citep{brailovskaya2024}. Our primary objective is to \emph{gain insights} into the behavior of the Jacobian under the different scenarios stated above, in the case of Multi-Layer Perceptrons (MLPs).

 Our analysis reveals a universality phenomenon for Jacobian products. In
particular, we show that suitably rescaled sparse networks and weakly correlated
Gaussian networks have the same fixed-depth large-width Jacobian norm as the
critically initialized iid Gaussian model. It is worth noting that the \emph{scaling factor depends on the pruning method} and that there exists a width-dependent correlation threshold that guarantees similar Jacobian properties to the i.i.d.\ case.\\

    Throughout this paper, we consider the Multi-Layer Perceptron (MLP) architecture given by
    \begin{equation}\label{eq:mlp_network}
        \texttt{MLP} \triangleright
        \begin{cases}
            Y_0(x) = W_{\text{in}}\,x, \\
            Y_k(x)
            = W_k\,\phi(Y_{k-1}(x)), \quad k=1,\dots, L, \\
            Y_{\text{out}}(x) = W_{\text{out}}\, \phi(Y_L(x)),
        \end{cases}
    \end{equation}
    where $x \in \reals^d$ is the input, $Y_{\text{out}} \in \reals^o$ is the network output, $W_{\text{in}} \in \reals^{n \times d}, W_k \in \reals^{n \times n}, W_{\text{out}} \in \reals^{o \times n}$ are the network weights, and~$\phi$ is the ReLU activation function given by $\phi(z) = \max(z, 0)$ (acting coordinate-wise). For the sake of simplification, we omit here the bias terms in the definition of the MLP. We refer to~$Y_k$ as the pre-activations or the features and $\phi(Y_k)$ as the activations. Hereafter, \emph{width} and \emph{depth} will be used to refer to~$n$ and~$L$, respectively.

    In practice, neural networks are usually trained with gradient-based algorithms such as Stochastic Gradient Descent (SGD), Adam, etc.~\citep{lecun2015deep,kingma2014adam,bottou2012stochastic}. This requires the calculation of the gradients of some loss function with respect to the weights $(W_k)_{1 \leq k \leq L}$ using back-propagation. Let~$\ell$ be a loss function (e.g.\ mean-squared error for regression, and cross-entropy loss for classification) and $\mathcal{D} = \{(x_i, z_i), i=1 \dots N\}$ be a fixed training dataset. DNN training aims to minimize the empirical objective
    $
    \loss(\mathbf{W}) = N^{-1} \sum_{i=1}^N \ell(Y_{\text{out}}(x_i), z_i),
    $
    where $\mathbf{W} = \{W_{\text{in}}, (W_k)_{1\leq k \leq L}, W_{\text{out}}\}$. With GD, the parameters $\mathbf{W}$ are updated with the rule:
    $$
    \mathbf{W} \leftarrow \mathbf{W} - \eta \frac{\partial \loss}{\partial \mathbf{W}}.
    $$
    For a datapoint~$(x,z)$, the gradient of the loss function evaluated at~$(x,z)$ w.r.t.\ the weights $W_{k}^{i,j}$ (for some $i,j \in \{1, \dots,n\}$) is given by:
    $$
    \frac{\partial \ell(Y_{\text{out}}(x), z)}{\partial W_k^{i,j}} = \frac{\partial \ell(Y_{\text{out}}(x), z)}{\partial Y_k^{i}(x)} \, \phi(Y_{k-1}^j(x)) = \frac{\partial \ell(Y_{\text{out}}(x), z)}{\partial Y_L(x)}^\top \frac{\partial Y_{L}(x)}{\partial Y_k^{i}(x)} \, \phi(Y_{k-1}^j(x)).
    $$
    Hence, the gradients inherently depend on the Jacobian terms $$J_k(x) = \frac{\partial Y_L(x)}{\partial Y_{k-1}(x)} = \left(\frac{\partial Y_L^i(x)}{\partial Y_{k-1}^j(x)}\right)_{1\leq i,j\leq n} \in \reals^{n \times n},$$
    for $k \in \{1,2,\dots, L\}$. Using the chain rule, it is easy to see that for $k \in \{1, \dots, L-1\}$, the Jacobian satisfies the recursion $
    J_k(x) = J_{k+1}(x) \times  W_{k}\, D_{k-1}(x)
    $,
    where $D_{k}(x) = \textup{Diag}(\phi'(Y_{k}(x))) \in \reals^{n \times n}$. Thus, we can express the Jacobian terms as a matrix product:
    \begin{equation}
        J_k(x) = \prod_{l=k}^L  W_l D_{l-1}(x), \quad k\in \{1, \dots, L-1\},\label{eq:jacobian-as-product}
    \end{equation}
    Hereafter, to alleviate the notation, we denote the Jacobian without reference to its input~$x$ and use $J_1 = \prod_{l=1}^L W_l D_{l-1}$, as defined earlier. We assume a fixed non-zero input~$x$ for the analysis. Nevertheless, as we show in our empirical results, our findings hold for randomly selected inputs from the dataset, demonstrating that our conclusions are independent of the input choice\footnote{It should be noted that while the input can impact stability in a DNN, this impact is often negligible if the input is normalized, with the architecture and weight distribution playing a more significant role.}.

    Due to the nature of the Jacobian (product of matrices), one can anticipate the occurrence of vanishing or exploding gradient phenomena as depth grows in instances where, e.g. the weights are improperly initialized. By examining the spectral norm of the Jacobian, denoted by~$\|J_1\|$ (i.e., the largest singular value of~$J_1$), as a function of the depth~$L$, distinct regimes can be identified wherein the Jacobian norm exhibits either exponential exploding or vanishing, or alternatively demonstrates a sub-exponential dependence relative to depth~\cite{pennington2017}. It is the exponential dependence on the depth that poses a practical problem, as it typically leads to fast degradation of the gradients (exponential vanishing) or numerical instability (exponential exploding). Hereafter, when the depth dependence is sub-exponential, the network is said to be \emph{stable}. For any varying quantities $a,b$, we will use the standard notation $b = \Theta(a), b = \mathcal{O}(a), b = \smallO(a)$, to mean respectively $\alpha a \leq b \leq \beta a$, $|b| \leq \beta |a|$ (where $\alpha, \beta >0$ are constants), and $\lim b/a = 0$. The notation $\Tilde{\Theta}$ and $\Tilde{\mathcal{O}}$ are used to suppress sub-exponential factors. Additionally, we write~$a\lesssim b$ to mean $a\leq \beta b$. When $a,b$ are random, the underlying constants may depend on the randomness. To stress that these constants may also depend on an extra parameter, typically the depth~$L$, we also write
    $\Theta_L,\tilde\Theta_L$, etc.

    \begin{definition}[Stable Jacobian]\label{def:stability}
        We say that the Jacobian of a network with a distribution~$q$ over the weights $\mathbf{W} \sim q$ is stable if:
        $$
        \E_{\mathbf{W} \sim q} \|J_1\|  = \tilde{\Theta}_L(1),
        $$
        where~$\|\cdot\|$ denotes the spectral norm, $\|A\| := \sqrt{\lambda_{\max}(AA^*)}$.
    \end{definition}
    In the fixed-depth infinite-width regime considered throughout this paper, the
Jacobian norm converges in probability to a deterministic limit under the
assumptions of our main results. For this reason, we shall formulate our universality results in probability, while continuing to use the above definition as the conceptual notion of stability. Here, we define stability for any weight distribution~$q$. Note that stability is defined as $\tilde{\Theta}(1)$ instead of~$\Theta(1)$, which hides sub-exponential terms. This is because empirical results suggest that for typical network depths (e.g., in the range of 10 to 100), sub-exponential dependence does not significantly affect the performance \citep{poole2016exponential, samuel2017deep, hayou2019impact}. In the following, we will analyze the fixed-depth infinite-width limit of the Jacobian.

    \section{Jacobian with i.i.d Weights at Initialization}\label{sec:jacobian_iid}
    Using some approximations,~\cite{pennington2017} showed that in the large width limit, when neural networks are initialized with independent Gaussian weights $\normal(0,\zeta)$, the largest singular value of the input-output Jacobian~$J_1$ can either exponentially explode or vanish with depth if the variance of the weights is different from $\zeta = 2/n$, which is also known as \emph{the Edge of Chaos} initialization. This choice of~$\zeta$ guarantees stability at initialization. This result holds under the following approximation.

    \begin{approximation}\label{approx:i.i.d_bernoulli}
        In the infinite-width limit, the diagonal entries of $(D_k)_{k \in 0, \dots, L}$ behave as i.i.d Bernoulli  variables with parameter~$1/2$. Moreover, they are independent of the weights~$\W$.
    \end{approximation}
    It is easy to see why \cref{approx:i.i.d_bernoulli} is a valid approximation in the large-width regime. When $n \to \infty$, it is well known that the entries of the pre-activations $(Y_k^i)_{1 \leq i \leq n}$ converge (in distribution) to i.i.d Gaussian random variables that are independent across~$i$ and~$k$ \citep{neal, lee_gaussian_process2018, matthews2018gaussian, hayou2019impact, yang2019tensor_i}. Moreover, these entries become independent of the weights~$\W$ in this limit. Hence, since the matrix~$D_k$ consists of diagonal elements of the form $\phi'(Y_k^i) = 1_{Y_k^i > 0}$, it holds that the diagonal elements become approximately i.i.d Bernoulli random variables with parameter~$1/2$ when~$n$ is large. We refer the reader to \cref{app:further_empirical_results} for an empirical verification of \cref{approx:i.i.d_bernoulli}.

    We state the following result using the approximate symbol "$\approx$" instead of "$=$" since the authors proved it under an approximation.
    \begin{figure}
        \begin{center}
            \includegraphics[width=0.27\linewidth]{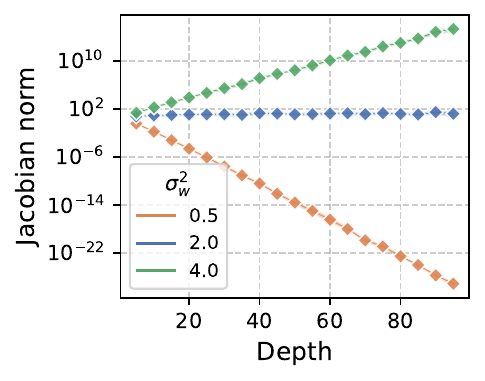}
            \includegraphics[width=0.3\linewidth]{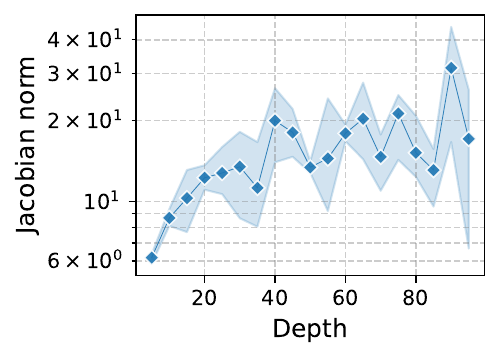}
        \end{center}
        \caption{\small{Illustration of the Jacobian norm at initialization in an MLP network of width $n=256$ and varying depth. The input is randomly selected from MNIST. All results are averaged over 3 runs. \textbf{(Left)} Impact of depth on the Jacobian norm for different~$\sigma_w$. \textbf{(Right)} Evolution of the Jacobian norm as a function of depth for critical initialization.}}\label{fig:jacnorm_init_i.i.d}
    \end{figure}
    \begin{theorem}[Corollary of Eq. (17) in~\cite{pennington2017}]\label{thm:criticality_i.i.d}
        Assume that the weights are i.i.d initialized as $W_k^{ij} \sim \normal(0,\sigma_w^2/n)$ for some $\sigma_w >0$. Then, in the limit $n \to \infty$, under \cref{approx:i.i.d_bernoulli}, we have the following:
        $$
        \|J_1\| \approx \Theta_L\left(L\, \left( \frac{\sigma_w^2}{2}\right)^L \right).
        $$
        As a result, the choice $\sigma_w^2 = 2$ guarantees stability.
    \end{theorem}

    The choice $\sigma_w^2 = 2$ corresponds to the Edge of Chaos initialization; an initialization scheme that allows deeper signal propagation in MLPs~\cite{poole2016exponential, samuel2017deep, hayou2019impact}. \cref{fig:jacnorm_init_i.i.d} illustrates the Jacobian norm for different choices of~$\sigma_w$ and depths~$L$. Exponential exploding/vanishing with depth can be observed in the non-critical initialization cases  $\sigma_w^2 \in \{0.5, 4.0\}$, while a sub-exponential growth w.r.t.\@ depth is achieved with the critical initialization $\sigma_w^2=2$ as predicted by \cref{thm:criticality_i.i.d}. In \cref{app:further_empirical_results}, we report the accuracy of trained networks for varying initialization schemes (resulting in different Jacobian norms) and further confirm that Jacobian stability is necessary to achieve non-trivial performance. This adds to the empirical evidence provided in~\cite{pennington2017, chhaibi2022free}.

    Our goal in this paper is to establish similar stability results in the two different scenarios previously stated (sparse networks and correlated weights). Hereafter, $J_1^{iid}$ denotes the Jacobian of an iid randomly initialized MLP with $\sigma_W^2=2$.
    In the next section, we introduce the core technical result in this paper, a general universality theorem for Jacobian products, based on recent results from random matrix theory.

\section{Universality of Jacobian products}
In this section, we prove a universality theorem asserting that a broad class of random matrices interleaved with Bernoulli gates have the same fixed-depth large-width Jacobian norm as the critically initialized iid Gaussian model. 
Throughout this section, $L\ge 1$ is fixed. Let
$D_0,\ldots,D_{L-1}$ be random diagonal matrices, independent of all weight matrices, and with entries i.i.d Bernoulli with parameter $1/2$. 
Let $G_1,\ldots,G_L$ be independent $n\times n$ Gaussian matrices with iid
entries
$$
    (G_l)_{ij}\sim \mathcal N(0,2/n),
$$
also independent of the $D_l$'s. 
For a collection of $n\times n$ random matrices $X_1,\ldots,X_L$, define
$$
    J_X:=X_LD_{L-1}X_{L-1}D_{L-2}\cdots X_1D_0,
$$
and
$$
    J_G:=G_LD_{L-1}G_{L-1}D_{L-2}\cdots G_1D_0.
$$

\begin{theorem}[Fixed depth universality]\label{th:jacobian-universality}
Assume that $L\ge 1$ is fixed. Let $X_1,\ldots,X_L$ be independent
$n\times n$ random matrices, independent of $D_0,\ldots,D_{L-1}$, whose
entries are centered and of variance $2/n$. Assume that one of the following two conditions holds. 
\begin{enumerate}
\item[(A)] The entries of each $X_l$ are independent, and
$$
r_n:= \left(
    \mathbb E
    \max_{\substack{1\le l\le L\\1\le i,j\le n}}
    |X_{l,ij}|^2
    \right)^{1/2}
$$
satisfies
$$
r_n\log^4 n\to0.
$$
\item[(B)] The matrices $X_l$ are centered Gaussian matrices and
$$
\rho_n:=
\max_{l}
\max_{(i,j)\ne(k,l)}
\left|
\operatorname{Cov}(X_{l,ij},X_{l,kl})
\right|
$$
satisfies
$$
n^{5/2}\rho_n\log n\to0.
$$
\end{enumerate}
Then
$$
\frac{\|J_X\|}{\|J_G\|}
    \xrightarrow[n\to\infty]{\mathbb P}1.
$$
\end{theorem}

To prove Theorem~\ref{th:jacobian-universality}, we use a layer-by-layer replacement argument. The key step is a one-layer comparison: after conditioning on all other layers and on the diagonal gates, the problem reduces to comparing
$$
    \|UXV\|
    \qquad\text{and}\qquad
    \|UGV\|,
$$
where $U,V$ are deterministic matrices, $X$ is the layer to be replaced, and
$G$ is the corresponding iid Gaussian layer. The next two lemmas provide this comparison in the two cases appearing in Theorem~\ref{th:jacobian-universality}. The first one applies to independent entries, while the second applies to weakly correlated Gaussian entries. Both lemmas makes use of recent breakthrough results of \citep{brailovskaya2024}. 
\begin{lemma}[One-layer universality for independent entries]\label{lem:one-layer}
    Let $U,V$ be $n\times n$ deterministic matrices. Let $X$ be an $n\times n$ random matrix whose
entries are independent, centered, and of variance $2/n$. Let $G$ be an $n\times n$ random matrix whose
entries are independent, centered Gaussian variables of variance $2/n$. Moreover let 
$$
r_n:= \left(
    \mathbb E
    \max_{i,j\le n}
    |X_{ij}|^2
    \right)^{1/2}.
$$
Assume $r_n\le 1$. Then  
$$
\mathbb P\Bigg( \big|\|UXV\|-\|UGV\|\big|\ge C\|U\|\|V\|
    \left[\sqrt{\frac{\log n}{n}}+r_n^{1/6}(\log n)^{2/3}+r_n^{1/2}\log n\right]\Bigg)
\le Cn^{-3}+Cr_n,
$$
where $C>0$ is a universal constant.
\end{lemma}
\begin{proof}
We use Hermitian dilation to reduce the non-self-adjoint matrix
$UXV$ to a self-adjoint one. Define
$$
H_X   :=\begin{pmatrix}
        0 & UXV\\
        V^*X^*U^* & 0
    \end{pmatrix},
    \qquad
H_G :=    \begin{pmatrix}
        0 & UGV\\
        V^*G^*U^* & 0
    \end{pmatrix}.
$$
Then
$$
\|H_X\|=\|UXV\|,\qquad \|H_G\|=\|UGV\|.
$$
Moreover,
$$
    H_X=\sum_{i,j=1}^n Z_{ij},
$$
where
$$  Z_{ij}  :=X_{ij}
    \begin{pmatrix}
        0 & Ue_i e_j^*V\\
        V^*e_j e_i^*U^* & 0
    \end{pmatrix}.
$$
The matrices $Z_{ij}$ are independent, centered, and self-adjoint. We now apply Theorem 2.8 of \cite{brailovskaya2024}. To do so, we note that $H_X$ is a sum of independent centered self-adjoint matrices, while $H_G$ is the Gaussian matrix with the same mean and covariance structure. We estimate the relevant parameters. We first note that 
$$
\|Z_{ij}\|
    \le
    |X_{ij}|\,\|U\|\,\|V\|.
$$
Therefore
\begin{equation}\label{eq: bound-Rbar}
    \bar R(H_X):=
    \left(
    \mathbb E\max_{i,j}\|Z_{ij}\|^2
    \right)^{1/2}
    \le
    \|U\|\|V\|\,r_n.
\end{equation}
Now note that 
$$ H_X^2 =    \begin{pmatrix}
        UXVV^*X^*U^* & 0\\
        0 & V^*X^*U^*UXV
    \end{pmatrix}.
$$
Since the entries of $X$ are independent, centered, and of variance $2/n$, we get
$$
\mathbb E[XVV^*X^*]=\frac{2}{n}\operatorname{Tr}(VV^*)I_n=
    \frac{2}{n}\|V\|_{\mathrm{HS}}^2I_n,
$$
and similarly
$$
\mathbb E[X^*U^*UX]=\frac{2}{n}\|U\|_{\mathrm{HS}}^2I_n.
$$
Thus
$$
\mathbb E H_X^2=
    \begin{pmatrix}
        \frac{2}{n}\|V\|_{\mathrm{HS}}^2UU^* & 0\\
        0 & \frac{2}{n}\|U\|_{\mathrm{HS}}^2V^*V
    \end{pmatrix}.
$$
Since $\|U\|_{\mathrm{HS}}\le \sqrt n\|U\|$ and
$\|V\|_{\mathrm{HS}}\le \sqrt n\|V\|$, it follows that
 \begin{equation}\label{eq:bound-sigma}
     \sigma(H_X):= \|\mathbb E H_X^2\|^{1/2}\le \sqrt2\,\|U\|\|V\|.
 \end{equation}  
Similarly, we have 
$$
\mathbb E[\langle x, UXV y\rangle^2] = \frac2n \|U^*x\|^2\|Vy\|^2 \leq \frac{2}{n} \|U\|^2\|V\|^2,
$$
for any unit vectors $x,y\in \mathbb{R}^n$. Therefore 
$$
\mathbb E[\langle x, H_X y\rangle^2]\leq \frac{8}{n} \|U\|^2\|V\|^2,
$$
for any unit vectors $x,y\in \mathbb{R}^{2n}$. 
Hence
$$
\sigma_*(H_X) :=
    \sup_{\|x\|=\|y\|=1}
    \left(
    \mathbb E|\langle x,H_Xy\rangle|^2
    \right)^{1/2}
    \le
    2\sqrt{2}\frac{\|U\|\|V\|}{\sqrt n}.
$$
Now set $R:= C_0\|U\|\|V\| \sqrt{r_n}$, for some sufficiently large universal constant $C_0$. Since $r_n\le 1$, then in view of \eqref{eq: bound-Rbar}, we have $ \bar R(H_X)\le R$. Together with \eqref{eq:bound-sigma} and the fact that $C_0$ is large enough, we have 
$$
R\ge \sqrt{\bar  R(H_X)  \sigma(H_X)} +
\sqrt 2\,\bar R(H_X).
$$
Therefore, we can apply \cite[Theorem 2.8]{brailovskaya2024} with $t=4\log (2n)$ to get  
$$
\mathbb P\left(
\big|\|UXV\|-\|UGV\|\big|\ge 
C\Big(\sigma_*(H_X)\sqrt t + R^{1/3}\sigma(H_X)^{2/3}t^{2/3} + Rt\Big), \, \max_{i,j}\|Z_{ij}\|\le R \right)
\le 2n e^{-t} \le n^{-3},
$$
for some universal constant $C$. 
Using the estimates obtained above,
$$
\sigma_*(H_X)\sqrt t\le 4\sqrt{2}\|U\|\,\|V\|
\sqrt{\frac{\log (2n)}{n}},
$$
$$
R^{1/3}\sigma(H_X)^{2/3}t^{2/3}
\le
C'\|U\|\,\|V\|
\,r_n^{1/6}(\log n)^{2/3},
$$
and
$$
Rt
\le
C'\|U\|\,\|V\|
\,r_n^{1/2}\log n,
$$
for some universal constant $C'$. 
Therefore, on the event $\{\max_{i,j}\|Z_{ij}\|\le R\}$, except with probability $n^{-3})$, 
$$
\big|\|UXV\|-\|UGV\|\big| \le
C''\|U\|\,\|V\|\left(\sqrt{\frac{\log n}{n}}
+r_n^{1/6}(\log n)^{2/3}+r_n^{1/2}\log n
\right),
$$
for some universal constant $C''$. It remains to control the event $\{\max_{i,j}\|Z_{ij}\|>R\}$. Using Markov's inequality and \eqref{eq: bound-Rbar}, 
$$
\mathbb P\left(\max_{i,j}\|Z_{ij}\|>R\right)
\le \frac{\mathbb E\max_{i,j}\|Z_{ij}\|^2}{R^2}
\le \frac{\|U\|^2\|V\|^2 r_n^2}{C_0^2\|U\|^2\|V\|^2 r_n}
=\frac{r_n}{C_0^2}.
$$
Putting the above together, we finish the proof. 
\end{proof}

\begin{lemma}[One-layer universality for  correlated Gaussian entries]
\label{lem:one-layer-correlated}
Let $U,V$ be $n\times n$ deterministic matrices. Let $X=(X_{ij})$ be a centered
Gaussian $n\times n$ matrix satisfying $\mathbb E X_{ij}^2=\frac2n$ for every \(i,j\). Let
$$
    \rho_n:=
    \max_{(i,j)\ne(k,l)}
    \left|
    \operatorname{Cov}(X_{ij},X_{kl})
    \right|.
$$ Let $G$ be an $n\times n$ random matrix whose
entries are independent, centered Gaussian variables of variance $2/n$. 
Then there exists a universal constant $C>0$ such that
$$
\mathbb P\left( \big|\|UXV\|-\|UGV\|\big|>C\|U\|\|V\|\left[\sqrt{\frac{\log n}{n}}+n^{5/4}\sqrt{\rho_n\log n}\right]\right)
\le Cn^{-3}.
$$
\end{lemma}
\begin{proof}
As in the proof of Lemma~\ref{lem:one-layer}, we use Hermitian dilation. 
Define
$$
    H_X:=
    \begin{pmatrix}
        0&UXV\\
        V^*X^*U^*&0
    \end{pmatrix},
    \qquad
    H_G:=
    \begin{pmatrix}
        0&UGV\\
        V^*G^*U^*&0
    \end{pmatrix}.
$$
Then $H_X$ and $H_G$ are centered self-adjoint Gaussian matrices of dimension
$2n$, and
$$
    \|H_X\|=\|UXV\|,
    \qquad
    \|H_G\|=\|UGV\|.
$$
We apply Proposition~8.2 of \cite{brailovskaya2024} to $H_X$ and $H_G$, and we now estimate the relevant parameters. 
Carrying a similar computation as in Lemma~\ref{lem:one-layer}, we have 
$$
\sigma_*(H_G) :=
    \sup_{\|x\|=\|y\|=1}
    \left(
    \mathbb E|\langle x,H_Gy\rangle|^2
    \right)^{1/2}\leq 2\sqrt{2} \frac{\|U\|\|V\|}{\sqrt n}.
$$
To estimate $\sigma_*(H_X)$, let $x,y\in \mathbb{R}^n$ be
unit vectors and write 
$$
    \langle y,UXVx\rangle
    =
    \sum_{i,j} X_{ij}c_{ij},
$$
where
$$
c_{ij}
    =
    \langle U^*y,e_i\rangle\langle e_j,Vx\rangle.
$$
The diagonal covariance contribution gives
$$
    \frac2n\sum_{i,j}|c_{ij}|^2
    \le
    \frac{2\|U\|^2\|V\|^2}{n}.
$$
For the off-diagonal covariance contribution, we use by Cauchy--Schwarz that
$$
    \sum_{i,j}|c_{ij}|
    \le
     n\|U\|\|V\|.
$$
Therefore
$$
\sum_{(i,j)\ne(k,l)}\left|\operatorname{Cov}(X_{ij},X_{kl})c_{ij}c_{kl}\right|
\le\rho_n n^2\|U\|^2\|V\|^2.
$$
It follows that
$$
    \sigma_*(H_X)
    \le
    2\|U\|\|V\|
    \left(
        \sqrt{\frac{2}{n}}+n\sqrt{\rho_n}
    \right).
$$

It remains to estimate the covariance distance $\Delta(H_X,H_G)$ appearing in
Proposition~8.2 of \cite{brailovskaya2024}. Let $M\in M_{2n}(\mathbb C)$ with
$\|M\|\le1$, and write
$$
    M=
    \begin{pmatrix}
        M_{11}&M_{12}\\
        M_{21}&M_{22}
    \end{pmatrix}.
$$
We estimate one block of
$
    \mathbb E[H_XMH_X]-\mathbb E[H_GMH_G],
$
the other blocks being treated identically. The upper-left block $\Gamma$ equals
$$
  U\left(\mathbb E[XAX^*]-\mathbb E[GAG^*]\right)U^*,
$$
where we denoted $VM_{22}V^*$ by $A$. We have 
$$
\Gamma_{ij}= \sum_{k,l} \left(
    \operatorname{Cov}(X_{ik},X_{jl})-
    \frac2n\mathbf 1_{\{i=j,\ k=l\}}
    \right)A_{kl}.
$$
Since the entries of $X$ have variance $2/n$, the term with $i=j$ and $k=l$ cancels, and therefore
$$
    |\Gamma_{ij}|
    \le
    \rho_n\sum_{k,l}|A_{kl}|.
$$
Using
$$
    \sum_{k,l}|A_{kl}|
    \le n \|A\|_{\mathrm{HS}}
    \le
    n^{3/2}\|A\|,
$$
we get
$$
    |\Gamma_{ij}|
    \le
    \rho_n n^{3/2}\|A\|.
$$
Consequently,
$$
    \|\Gamma\|
    \le
    n\max_{i,j}|\Gamma_{ij}|
    \le
    \rho_n n^{5/2}\|A\|.
$$
Since $\|VM_{22}V^*\|\le \|V\|^2$, the upper-left block is bounded in norm by
$$
    \rho_n n^{5/2}\|U\|^2\|V\|^2.
$$
The same estimate applies to the other blocks and we get
$$
    \Delta(H_X,H_G)
    \le
    C\rho_n n^{5/2}\|U\|^2\|V\|^2,
$$
for some universal constant $C$. Letting $x=2\sqrt{\log(2n)}$ and gathering the previous estimates, we deduce that
$$
  \Delta(H_X,H_G)^{\frac12}\sqrt{\log(2n)}
  +\bigl(\sigma_*(H_G)+\sigma_*(H_X)\bigr)x
  \le C'\|U\|\|V\|\left(\sqrt{\frac{\log n}n}+n^{5/4}\sqrt{\rho_n\log n}\right),
$$
for some universal constant $C'$.
We finish the proof by applying Proposition~8.2 of \cite{brailovskaya2024}.
\end{proof}

\begin{proof}[Proof of Theorem~\ref{th:jacobian-universality}]
We prove the theorem under assumption (A); the other case follows verbatim by substituting the use of Lemma~\ref{lem:one-layer} by Lemma~\ref{lem:one-layer-correlated}. Define $J^{(0)}:=J_X$ and, for $1\le r\le L$,
$$
J^{(r)}:=X_LD_{L-1}\cdots X_{r+1}D_r G_rD_{r-1}G_{r-1}D_{r-2}\cdots G_1D_0.
$$
Thus $J^{(L)}=J_G$. By the triangle inequality,
\begin{equation}\label{eq:telescoping}
    \big|\|J_X\|-\|J_G\|\big|\le \sum_{r=1}^L \big|\|J^{(r-1)}\|-\|J^{(r)}\|\big|.
\end{equation}
Fix $r\in\{1,\ldots,L\}$. Conditionally on all matrices except $X_r$ and $G_r$, we may write
$$
    J^{(r-1)}=U_rX_rV_r,
    \qquad
    J^{(r)}=U_rG_rV_r,
$$
where
$$
    U_r:=X_LD_{L-1}\cdots X_{r+1}D_r
$$
and
$$
    V_r:=D_{r-1}G_{r-1}D_{r-2}\cdots G_1D_0.
$$
The matrices $U_r,V_r$ are deterministic under this conditioning. We first show that these matrices are bounded with high probability. For $K>0$, define
$$
\Omega_K:=
\left\{\max_{1\le l\le L}\|X_l\|\le K,\quad \max_{1\le l\le L}\|G_l\|\le K\right\}.
$$
Since $\|D_l\|\le1$, on \(\Omega_K\) we have
$$
    \|U_r\|\le K^{L-r},
    \qquad
    \|V_r\|\le K^{r-1},
$$
and hence
$$
    \|U_r\|\|V_r\|\le K^{L-1}.
$$
On \(\Omega_K\), Lemma~\ref{lem:one-layer}, conditionally on the other layers and the diagonal matrices, gives
$$
    \mathbb P\!\left(\big|\|J^{(r-1)}\|-\|J^{(r)}\|\big|\ge C K^{L-1}
    \left(\sqrt{\frac{\log n}{n}}+r_n^{1/6}(\log n)^{2/3}+r_n^{1/2}\log n\right)\,\middle|\, U_r,V_r\right)
\le Cn^{-3}+Cr_n,
$$
for some universal constant $C$. 
Consequently,
$$
\mathbb P\left(
\big|\|J^{(r-1)}\|-\|J^{(r)}\|\big|\ge C K^{L-1}
    \left(\sqrt{\frac{\log n}{n}}+r_n^{1/6}(\log n)^{2/3}+r_n^{1/2}\log n\right)\right)
\le \mathbb P(\Omega_K^c)+ Cn^{-3}+Cr_n. 
$$
Using \eqref{eq:telescoping} and a union bound over $r=1,\ldots,L$, we get
\begin{equation}\label{eq:jacobian-diff}
\mathbb P\left(
\big|\|J_X\|-\|J_G\|\big|\ge C L K^{L-1}
    \left(\sqrt{\frac{\log n}{n}}+r_n^{1/6}(\log n)^{2/3}+r_n^{1/2}\log n\right)\right)
\le L\mathbb  P(\Omega_K^c)+ CLn^{-3}+CLr_n. 
\end{equation}
By standard operator norm estimates for Gaussian matrices (see, \cite[Theorem~4.4.3]{Vershynin2018}), we have 
$$
\mathbb P\left(\max_{1\le l\le L}\|G_l\|\ge K/2\right)
\le C'e^{-cn},
$$
for some constants $C', c$ depending only on $L$ and $K$. On the other hand, applying Lemma~\ref{lem:one-layer} to each $X_l$ with $U=V=I$, we deduce by union bound that 
\begin{equation}\label{eq:norm-concentration}
    P(\Omega_K^c)
    \leq\mathbb P\left(\max_l\|G_l\|>\frac K2\right)+\sum_{l=1}^L\mathbb P(|\|X_l\|-\|G_l\|| > K/2)
    \leq C''(e^{-cn}+ n^{-3}+r_n),
\end{equation}
for some constant $C''$ depending only on $L$ and $K$. Since  $r_n\log^4 n\longrightarrow0$, then combining \eqref{eq:jacobian-diff} and \eqref{eq:norm-concentration}, we deduce that 
$$
\big|\|J_X\|-\|J_G\|\big|
\xrightarrow[n\to\infty]{\mathbb P}0.
$$
It remains to combine the above with Theorem~\ref{thm:criticality_i.i.d} to finish the proof. 
\end{proof}

    \section{Jacobian in Sparse Networks}
    Network pruning aims at removing redundant weights that do not significantly affect model performance \citep{lecun_pruning}. Such weights are identified using some pruning criterion that determines the importance  of each weight in the network.
    By removing these weights, one could significantly reduce the computational requirements for both the training and deployment of DNNs \citep{hassibi_pruning}. After pruning, the network becomes \emph{sparse}, and training such networks has been proven challenging in practice \citep{franklelottery2019}. In this section, we analyze the Jacobian norm of sparse networks obtained by random pruning and magnitude-based pruning and provide the necessary conditions for stability. Notably, we show that a simple scaling trick ensures depth stability of the Jacobian. We also identify an \emph{edge of stability} for sparsity and support our theory with empirical evidence. Our current analysis is limited to \emph{pruning at initialization} \citep{lee_snip2018, wang2022recent, hayou2021robust} of MLPs, i.e. pruning performed on i.i.d weights. Extending this theory to modern architectures is not straightforward and is an interesting question for future work.\\

    \noindent\textbf{Pruning.} Consider the MLP architecture described in \cref{eq:mlp_network}. Pruning involves the application of a mask $\mask \in[0,\beta_p]^p$ (where~$p$ is the total number of parameters in the network)
    to the weights of the network~$\W$ producing a \emph{pruned} network with weights~$\W^B$, where $\W^B= \mask \odot \W$ is the Hadamard (i.e.,  element-wise) product.
    We say weight $i\in [p]$ is pruned if $b_i = 0$. This can be performed via different procedures. A standard approach to generating masks is to compute a score $g_k^{ij}$ for each weight $W_k^{ij}$ according to some criterion.
    The mask is then created by keeping the top~$m$ weights by score, where~$m$ is chosen to meet some desired sparsity level~$s$ (fraction of weights to remove). The Jacobian of the pruned network~$J^B_1$ is given by:
    $$
    J^B_1 = \prod_{l=1}^L\W^B_l D_{l-1}.
    $$
    We propose to study the Jacobian norm~$\|J^B_1\|$ for networks pruned at initialization with two different pruning methods:
    \begin{enumerate}[leftmargin=*]
        \item {\bf Random Pruning}: weights are randomly pruned with probability~$s$ (the sparsity).
        \item {\bf Score-Based Pruning}: weights are scored using a certain criteria (e.g.,  magnitude, sensitivity,). In this work, we consider magnitude-based pruning for tractability. Extending our results to other score-based pruning algorithms such as sensitivity-based pruning (which uses gradients) is not trivial and will require different proof machinery.
    \end{enumerate} 
    The results of this section are applications of the universality theorem (\cref{th:jacobian-universality}). The key step is to identify the correct rescaling of the retained weights so that the effective entries have the same second moment as the critically initialized Gaussian model. Once this variance matching is established, stability follows by verifying the maximal-entry condition appearing in \cref{th:jacobian-universality}. The main takeaway is that scaling the weights is required to maintain stability in pruned networks, and that the scaling factor depends on the pruning method.

    \subsection{Random Pruning}
    In the subsequent analysis, we use the notation $a_n \gg b_n$ for two positive sequences $a_n, b_n$, whenever $ b_n = \smallO(a_n)$.

    Recall that $J_1^{iid}$ refers to the Jacobian of an iid randomly initialized MLP as in \cref{sec:jacobian_iid}. The following result shows that using a scaling trick, the Jacobian of the pruned network is asymptotically identical to that of non-pruned iid randomly initialized network.
    \begin{theorem}[Scaling guarantees stability]\label{thm:fixed_sparsity}
        Consider random pruning with sparsity level $s_n \in (0,1)$ that can either depend on~$n$ or be constant. Set the mask $B:=(1 - s_n)^{-1/2} B_{i,j}$ where the $B_{i,j}$'s are iid Bernoulli$(1-s_n)$ variables. Then, under the assumption that $1 - s_n \gg \frac{\log^9n}{n}$, the Jacobian of the scaled sparse network given by $
        J^B_1 =\prod_{l=1}^L  (B_l \odot W_l) D_{l-1} $ satisfies
$$
\frac{\|J_1^B\|}{\|J_1^{iid}\|}\xrightarrow[n\to\infty]{\mathbb P}1.
$$
As a result, under these conditions, the Jacobian of the pruned network is also stable.
    \end{theorem}
\begin{proof}
It is easy to check that the entries of $X_l=B_l \odot W_l$ are independent, centered and of variance $2/n$. Thus the variance matching condition of
\cref{th:jacobian-universality} is satisfied and it remains to verify the maximal-entry condition. Since
$$
|X_{l,ij}|\le (1-s_n)^{-1/2}|W_{l,ij}|,
$$
we obtain
$$
r_n\le (1-s_n)^{-1/2}
\left(\mathbb E\max_{l,i,j} |W_{l,ij}|^2\right)^{1/2}.
$$
Using a standard estimate on the maximum of Gaussian variables (see for instance~\cite[Exercise~2.5.10]{Vershynin2018} or~\cite[Lemma~5.2]{vanHandel16}), we have 
$$
\left(\mathbb E \max_{l,i,j}|W_{l,ij}|^2\right)^{1/2}
\lesssim \sqrt{\frac{\log n}{n}}.
$$
Hence
$$
r_n
\lesssim
\sqrt{\frac{\log n}{n(1-s_n)}}.
$$
The assumption
$$
1-s_n\gg \frac{\log^9 n}{n}
$$
ensures that
$$
r_n\log^4 n\to0.
$$
The result now follows from \cref{th:jacobian-universality}.
\end{proof}

    \cref{thm:fixed_sparsity} reveals that, once pruning has been performed, re-scaling is essential for stabilizing the Jacobian. Specifically, if we don't scale the weights (or equivalently the mask) with $(1-s_n)^{-1/2}$ then the theorem can be stated as having~$\|J_1^B\|$ asymptotically proportional to $\Tilde{\Theta}_L((1-s_n)^{L/2})$.\footnote{Recall that $\|J_1^{iid}\| = \Tilde{\Theta}_L(1)$ which yields the stated result.} This exponential dependence on depth indicates that stability cannot be attained without scaling the weights. In other words, when starting with a critically initialized network, the weights must be re-scaled after pruning to account for the resulting sparsity. This scaling factor $(1-s_n)^{-1/2}$ makes the infinite-width behavior of the spectral norm similar to that of a non-pruned critically initialized neural network, thereby guaranteeing stability in the sparse network, as evidenced in \cref{fig:jacnorm_depth_prunning}. Further experiments are provided in \cref{app:further_empirical_results}.

    In \cref{thm:fixed_sparsity}, the sparsity~$s_n$ can depend on~$n$. We provide an upper bound on the sparsity in terms of the width, in order for the stability result to hold. This result allows us to identify an \emph{edge of stability}, defined as a maximal sparsity (in terms of~$n$) so that the stability holds.
    This result highlights an interesting \textit{phase transition} phenomenon with respect to the sparsity. When the sparsity is of order $1 - s_n \sim n^{-1}$ up to a logarithmic factor, the result of \cref{thm:fixed_sparsity} no longer holds, and the spectral norm of the Jacobian is significantly different from that of a non-pruned critically initialized network, see e.g.~\cite{MR3945756, MR4234995}. In \cref{fig:jacnorm_depth_prunning},  We observe this behavior when the sparsity hits the level~$99\%$, which is of order $\log_{10}(256)/256 \approx 0.009$. It is worth noting that the condition $1 - s_n \gg n^{-1} \log^9n$ is a sufficient condition, and it is likely that phase transition occurs at a smaller threshold.

    \begin{figure}
        \centering
        \begin{center}
            \includegraphics[width=0.3\textwidth]{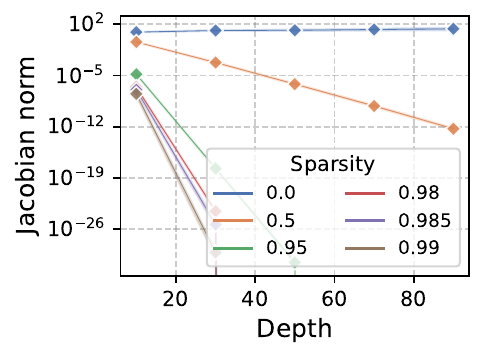}
            \includegraphics[width=0.3\textwidth]{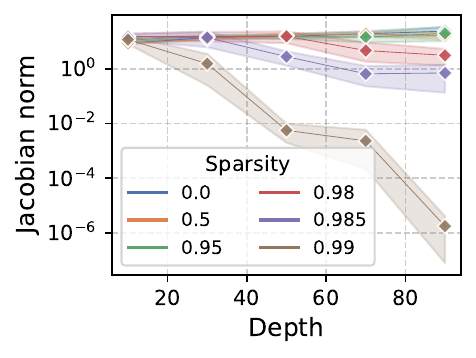}
        \end{center}
        \caption{\small{Jacobian norm after pruning at initialization as depth increases in a randomly pruned MLP of width $n = 256$ (input sampled randomly from MNIST). \textbf{(Left)} Without scaling. \textbf{(Right)} With scaling from \cref{thm:fixed_sparsity}.}}
        \label{fig:jacnorm_depth_prunning}
    \end{figure}
    \subsection{Score-Based Pruning}

    To complement the previous section, we show that a similar stability result holds with score-based pruning, however, the scaling factor is different in this case. We restrict our analysis to magnitude-based pruning, performed at initialization, where weights are scored based on their magnitude $g_w = |w|$~\cite{han2015learning}.  The result can in principle be extended to other score-based methods such as gradient-based pruning \citep{lee_snip2018}, but it will require different proof machinery.
    The main takeaway is that the scaling factor depends on the pruning method. \\
    \begin{theorem}[Magnitude-based pruning, deterministic threshold]\label{thm:magnitude-based-deterministic}
        Suppose that the coefficients of the weight matrix~$\W$ are iid $\normal(0,\frac2n)$ variables,
        and consider magnitude-based pruning in which all coefficients above a threshold $t\le\sqrt{\frac{4\log n-c\log\log n}n}$ in absolute value
        are kept, with $nt^2\to\infty$ and $c>36$. Specifically, choose the mask $B:=(b_{i,j})$ defined by
        \begin{equation}
          b_{i,j}:={\left(\sqrt\frac\pi{nt^2}\,e^{nt^2/4}\right)}^{\frac12}\mathbf1_{\{|w_{i,j}|>t\}},
          \label{eq:expression-mask}
        \end{equation}
        Then, the Jacobian of the scaled sparse network given by $
        J^B_1 =  \prod_{l=1}^L  (B_l \odot W_l) D_{l-1},$ satisfies,
$$
\frac{\|J_1^B\|}{\|J_1^{iid}\|}\xrightarrow[n\to\infty]{\mathbb P} 1.
$$
        As a result, under these conditions, the Jacobian of the pruned network is also stable.
    \end{theorem}
    \begin{proof}
Let
$$
a_n :=\left(\sqrt{\frac{\pi}{nt^2}}e^{nt^2/4}\right)^{1/2},
$$
so that $b_{ij}=a_n\mathbf 1_{\{|w_{ij}|>t\}}$. We first replace $a_n$ by the exact variance-normalizing constant
$$
    \widehat a_n:=\left(\frac{2}{n \mathbb E\left[
        |w_{ij}|^2\mathbf 1_{\{|w_{ij}|>t\}}
    \right]}\right)^{1/2}.
$$
Then
$$
\widehat X_{ij} :=  \widehat a_n w_{ij}\mathbf 1_{\{|w_{ij}|>t\}}
$$
is centered, independent over $i,j$, and satisfies
$$
    \mathbb E|\widehat X_{ij}|^2=\frac2n.
$$    
It remains to check the maximal-entry condition in Theorem~\ref{th:jacobian-universality}. Standard estimates on Gaussian tails imply that 
$$
\mathbb E\left[|w_{ij}|^2\mathbf 1_{\{|w_{ij}|>t\}}\right]\sim \frac{2t}{\sqrt{n\pi}}\,e^{-nt^2/4},
$$
as $nt^2\to \infty$. Consequently, $\widehat a_n^2 \sim a_n^2$ and since
$$
    nt^2\le 4\log n-c\log\log n,
$$
we have
$$
\widehat a_n^2\lesssim a_n^2  \lesssim n(\log n)^{-c/4}.
$$
Therefore, using the standard Gaussian maximum estimate and the fact that $L$ is fixed, 
$$
\begin{aligned}
r_n &:=
\left(\mathbb E\max_{l,i,j}|\widehat X_{l,ij}|^2\right)^{1/2}  \\
&\le \widehat a_n \left(\mathbb E\max_{l,i,j}|w_{l,ij}|^2\right)^{1/2}
\lesssim (\log n)^{1/2-c/8}.
\end{aligned}
$$
Since $c>36$, it follows that $r_n\log^4 n\to0$. 
Thus $\widehat X_l$ satisfies the assumptions of Theorem~\ref{th:jacobian-universality}. Therefore, 
$$
 \frac{\|\widehat J_1\|}{\|J_1^{iid}\|}
    \xrightarrow[n\to\infty]{\mathbb P}
    1,
$$
where 
$$
\widehat J_1:=\prod_{l=1}^L \widehat X_l D_{l-1}.
$$
Since $L$ is fixed and $\widehat a_n/a_n\to1$, we write 
$$
\frac{\|J_1^B\|}{\|J_1^{iid}\|}= \left(\frac{a_n}{\widehat a_n}\right)^L  \frac{\|\widehat J_1\|}{\|J_1^{iid}\|}\xrightarrow[n\to\infty]{\mathbb P}
    1,
$$
and finish the proof. 
    \end{proof}

    In score-based pruning, we
    compute some scores $g \in \reals^p$ and prune based on the value of~$g$: keep the $(1-s)\times p$ weights with the highest scores, where~$n$ is the total number of parameters in the matrix.\footnote{In general, magnitude-based pruning can be performed on a matrix level or a model level. For the sake of simplification, we restrict our analysis to the matrix level.}
    This introduces a dependence on the coefficients, which becomes weaker when the number of parameters~$p$ gets larger. Asymptotically (in~$p$), score-based pruning becomes a rejection algorithm.

    \begin{theorem}[Magnitude-based pruning, random threshold]\label{thm:magnitude-based}
        Suppose that the coefficients of the weight matrix~$\W=(w_{i,j})$ are iid $\normal(0,\frac2n)$ variables,
        and consider magnitude-based pruning in which the $r\ge n\log^c n$ largest weight coefficients in absolute value
        are kept, for some $c>9$, and assume $r=o(n^2)$. Specifically, let $t:=t(n)>0$ solve
        \[\P(|w_{i,j}|\ge t)=\frac r{n^2}\]
        and choose the mask $B':=(b'_{i,j})$ defined by
        \[b'_{i,j}:={\left(\sqrt{\frac\pi{nt^2}}e^{nt^2/4}\right)}^{\frac12}\mathbf1_{\{|w_{i,j}|\ge w^{(r)}\}},\]
        where $w^{(1)}\ge w^{(2)}\ge\cdots\ge w^{(n^2)}$
        is the non-increasing rearrangement of $\{|w_{i,j}|:i,j\le n\}$.
        Then, the Jacobian of the scaled sparse network given by $
        J^{B'}_1 =  \prod_{l=1}^L  (B_l' \odot W_l) D_{l-1},$ satisfies,
        $$
\frac{\|J_1^{B'}\|}{\|J_1^{iid}\|}
    \xrightarrow[n\to\infty]{\mathbb P}
    1.$$
        As a result, under these conditions, the Jacobian of the pruned network is also stable.
    \end{theorem}
    \begin{proof}
        Note that $t=\frac2{\sqrt n}\erf^{-1}(1-\frac r{n^2})$, where $\erf(z):=\frac2{\sqrt\pi}\int_0^z e^{-x^2}dx$ denotes the standard error function. Since we have $\erf^{-1}(1-q)\le\sqrt{-{\log}q}$ for every $q\in(0,1]$ (see, e.g., \cite{pollak56}), the assumption on~$r$ ensures that $nt^2\to\infty$ and
        \[t\le\frac2{\sqrt n}\sqrt{\log{\frac{n^2}r}}\le\sqrt{\frac{4\log n-4c\log\log n}n},\]
        for $4c>36$ and~$n$ sufficiently large, so we may compare the random-threshold pruning $\mathbf W\odot B'$ to the model~$\mathbf W\odot B$ of \cref{thm:magnitude-based-deterministic} with the deterministic threshold~$t$.
        We will thus show that
    \begin{equation}\label{eq:exp-pruning diff}
        \lim_{n\to\infty}\E\|\W\odot B'-\W\odot B\|=0.
    \end{equation}
      By Markov's inequality, this implies
      $$
    \|\W\odot B'-\W\odot B\|
    \xrightarrow[n\to\infty]{\mathbb P}
    0.
$$
Since
$$
\|J_k^B-J_k^{B'}\|\le\|W_k\odot B'_k-W_k\odot B_k\|\,\|D_{k-1}\|\,\|J_{k+1}^{B'}\|
+
\|W_k\odot B_k\|\,\|D_{k-1}\|\,\|J_{k+1}^B-J_{k+1}^{B'}\|,
$$
submultiplicativity of the norm and a straightforward backward induction on $k$ show that
$$
    \|J_1^B-J_1^{B'}\|
    \xrightarrow[n\to\infty]{\mathbb P}
    0.
$$
Consequently,
$$
    \big|
        \|J_1^B\|-\|J_1^{B'}\|
    \big|
    \xrightarrow[n\to\infty]{\mathbb P}
    0,
$$
which combined with \cref{thm:magnitude-based-deterministic}, yields the desired result. 

The remainder of the proof aims to establish \eqref{eq:exp-pruning diff}. Factoring out the common scaling
\[{\left(\sqrt{\frac{\pi}{nt^2}}\, e^{nt^2/4}\right)}^{\frac12}
\lesssim\sqrt{n\log^{-c}n},\]
we have
      \[\|\W\odot B'-\W\odot B\|\lesssim\|\Delta\|\sqrt{n\log^{-c}n}\]
       where $\Delta:=(\Delta_{i,j})$ is the $n\times n$ matrix with coefficients
       \[\Delta_{i,j}:=\left(\mathbf1_{\{|w_{i,j}|\ge w^{(r)}\}}-\mathbf1_{\{|w_{i,j}|\ge t\}}\right)w_{i,j}.\]
        Since the law of $\Delta$ is invariant by permutation of its entries, we have
        by \cite[Theorem~1.2]{seginer2000} that
        \[ \E\|\Delta\|\lesssim  \E\max_{1\le i\le n}|R_i|+ \E\max_{1\le j\le n}|C_j|,\]
        where $|R_i|$ (resp.\ $|C_j|$) denotes the Euclidean norm of $\Delta$'s $i$-th row (resp., $\Delta$'s $j$-th column). By symmetry between rows and columns, it follows by Jensen's inequality that
        \[\E\|\Delta\|\le2\sqrt{\E{\max_{1\le i\le n}|R_i|^2}}.\]
        Let \[J(i):=\Bigl\{j:w^{(r)}>|w_{i,j}|\ge t\enspace\text{or}\enspace t>|w_{i,j}|\ge w^{(r)}\Bigr\},\]
       so that, by Hölder inequality,
       \[\E{\max_{1\le i\le n}|R_i|}^2=\E{\max_{1\le i\le n}\sum_{j\in J(i)}{|w_{i,j}|}^2}\le\left(\E{\max_{1\le i\le n}\card J(i)^{\frac{\alpha}{\alpha-1}}}\right)^{1-\frac1{\alpha}}
       \left(\E{\max_{1\le i,j\le n}|w_{i,j}|^{2\alpha}}\right)^{\frac1{\alpha}}\]
       for all $\alpha>1$.
  Using standard concentration estimate on the maximum of Gaussian variables (see for instance~\cite[Lemma~5.2]{vanHandel16}) with $\alpha:=\log n$, we obtain
        \[\E{\|\Delta\|}\lesssim\sqrt{\frac{\log n}n\E{\max_{1\le i\le n}\card J(i)}}
        \]
        and therefore
        \begin{equation}\label{eq:max-card}
            \E{\|\W\odot B'-\W\odot B\|}\lesssim\sqrt{\E{\max_{1\le i\le n}\card J(i)}}\log^{\frac{1-c}{2}}n.
        \end{equation}
        We now estimate $\E{\max_{1\le i\le n}\card J(i)}$. The sum $S:=\sum_{i=1}^n\card J(i)$
        counts how many $|w_{i,j}|$ land in the interval $[\min(t,w^{(r)}),\max(t,w^{(r)}))$.
        Recalling that $\P(|w_{i,j}|\ge t)=\frac r{n^2}$, it is straightforward to check that
        it corresponds to the absolute deviation $S=|N-r|$ of the Binomial($n^2,\frac r{n^2}$)
        random variable\[N:=\sum_{1\le i,j\le n}\mathbf1_{\{|w_{i,j}|\ge t\}}.\]
        Observe that, conditional on~$S$, the $|w_{i,j}|$'s remain exchangeable so the slices $(\card J(1),\ldots,\card J(n))$
        are obtained by sampling without replacement~$S$ indices from $\{1,\ldots,n\}^2$ and counting how many of these indices belong to each
        row~$1\le i\le n$. This is equivalent to a balls-and-bins problem (or multivariate hypergeometric distribution) with $S$ balls drawn without replacement from an urn containing $n^2$ balls, $n$ of which are of color~$i$ for every $1\le i\le n$. By
        \cite[Example 3.1(c)]{joagdev83}, the vector $(\card J(1),\ldots,\card J(n))$ is therefore negatively associated, which
        by \cite[Proposition~5]{dubhashi98} implies that each $\card J(i)$ satisfies the classical Bernstein inequality:
        \[\P\left(\card J(i)>\frac Sn+k\:\middle|\:S\right)
        \le\exp\left(-\frac{k^2/2}{\frac Sn+k/3}\right).\]
        By union bound,
        \[\P\left(\max_{1\le i\le n}\card J(i)>\frac Sn+k\:\middle|\:S\right)
        \le n\exp\left(-\frac{k^2/2}{\frac Sn+k/3}\right),\]
        and then
        \begin{align*}
        	\E\left[\max_{1\le i\le n}\card J(i)\:\middle|\:S\right]
        	&=\sum_{k=0}^\infty\P\left(\max_{1\le i\le n}\card J(i)> k\:\middle|\:S\right)\\[.4em]
        	&\le\frac{4S}n+2\log n+\sum_{k>\max(\frac{3S}n,2\log n)}n\exp\left(-\frac{k^2/2}{\frac Sn+k/3}\right)\\[.4em]
        	&\le\frac{4S}n+2\log n+\sum_{k>2\log n}n\exp\left(-\frac34k\right)\\[.4em]
        	&\lesssim\frac Sn+\log n.
        \end{align*}
        Now, recall that $\E[S]=\E{|N-\E[N]|}\le\sqrt{\var(N)}=\sqrt{r(1-\frac r{n^2})}\le n$. Taking expectations, we find
        \[\E\left[\max_{1\le i\le n}\card J(i)\right]\lesssim\log n.\]
        Plugging this back into \eqref{eq:max-card}, we deduce
        \[\E{\|\W\odot B'-\W\odot B\|}\lesssim\sqrt{\log^{\frac32-c}n},\]
       which tends to~$0$ since $c>9$.
    \end{proof}



    \begin{figure}
        \begin{center}
            \includegraphics[width=0.255\textwidth]{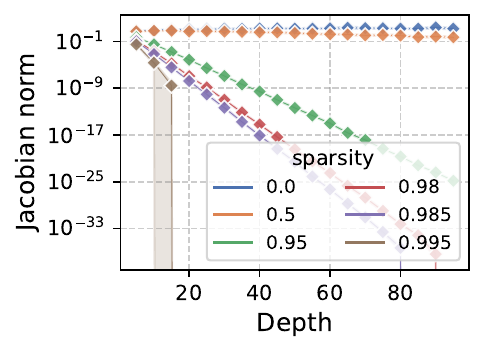}
            \includegraphics[width=0.249\textwidth]{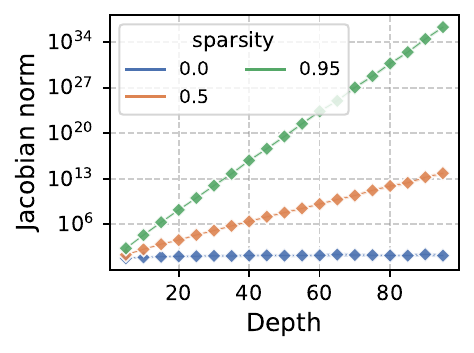}\\
            \includegraphics[width=0.245\textwidth]{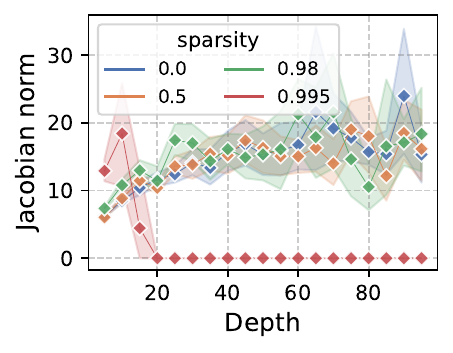}
            \includegraphics[width=0.249\textwidth]{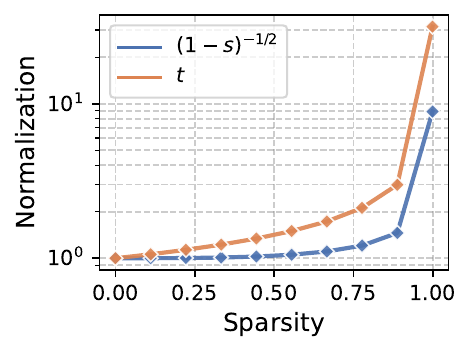}
        \end{center}
        \caption{\small{Jacobian norm after pruning at initialization as depth increases in a \emph{magnitude-based} pruned MLP of width $n = 256$ (input sampled randomly from MNIST). \textbf{(Upper Left)} Without scaling (scaling factor = 1). \textbf{(Upper Right)} scaling factor $= (1 - s_{n})^{-\frac{1}{2}}$. \textbf{(Lower Left)} Scaling factor is calculated based on Theorem~\ref{thm:magnitude-based}. \textbf{(Lower Right)} Comparison of the scaling factors as a function of sparsity level.}}
        \label{fig:jacnorm_depth_mag_prunning}
    \end{figure}

    In \cref{fig:jacnorm_depth_mag_prunning}, we demonstrate the effect of the scaling factor on the evolution of the Jacobian norm with depth. The top two plots show that without scaling (scaling factor = 1), the Jacobian norm is not stable and remains so when using the same scaling factor as in the case of random pruning ($(1 - s_{n})^{-\frac{1}{2}}$). The stability is recovered with the correct scaling factor stated in \cref{thm:magnitude-based} which is shown in the third plot. The bottom plot compares the  two scaling factors, used in the second and third plots respectively, in function of the sparsity. The takeaway is that the optimal scaling factor differs depending on the pruning criteria. The optimal scale for random pruning does not guarantee stability for magnitude-based pruning. It is primordial to determine the correct scaling factor that would guarantee the stability with different pruning criteria. Note that our results on sparse networks confirm previous results on weight scaling in the case of sensitivity-based pruning~\cite{hayou2021robust}.\footnote{It is worth mentioning that the stability measure used in~\cite{hayou2021robust} is based on the second norm of the gradient, which is a weaker measure than the Jacobian spectral norm.} We refer the reader to \cref{app:further_empirical_results} for results on the performance on trained sparse networks with and without scaling.

    In the next section, we show a second application of our generic stability result (\cref{th:jacobian-universality}). We study the stability of the Jacobian when the weights are not necessarily independent.

    \section{Jacobian with Dependent Weights}\label{app:further_theoretical_results}
    
    The reader might ask what happens to the Jacobian norm when the weights are \emph{not} independent, or equivalently: \emph{what level of correlation may be introduced between the weights without jeopardizing Jacobian stability?}

    The next theorem demonstrates that weight matrices~$\W$ with correlated entries can still yield Jacobian stability, provided that the correlation does not surpass $o(n^{-3/2}/ \log(n))$ asymptotically. In this case, the Jacobian spectral norm has the same fixed-depth large-width limit as in the iid Gaussian case.

    \begin{theorem}[Stability with Dependent Weights]\label{thm:criticality_non_i.i.d}
        Assume that the weights $W_1, W_2, \dots, W_L$ are independent copies of $n\times n$ weight matrix~$W$ consisting of centered Gaussian entries with variance~$2/n$ and correlation $\textup{corr}(W^{ij}_k, W^{m l}_k) = o(n^{-3/2}/\log(n))$ if $(i,j) \neq (m,l)$. Then, 
        $$
      \frac{\|J_1\|}{\|J_1^{iid}\|}
\xrightarrow[n\to\infty]{\mathbb P}
1.
        $$
    Consequently, the correlated model inherits the same fixed-depth large-width Jacobian stability as the iid model.    
\end{theorem}
    \begin{proof}
Since $\var(W_{ij})=\frac2n$, the correlation assumption is equivalent to
$$
\max_{(i,j)\neq(m,k)}\bigl|\operatorname{Cov}(W_{ij},W_{mk})
\bigr|=o\!\left(
\frac{1}{n^{5/2}\log n}\right).
$$
Therefore the correlated Gaussian condition (B) of
\cref{th:jacobian-universality} is satisfied. The conclusion follows
immediately from that theorem.
    \end{proof}
    \begin{figure}
        \begin{center}
            \includegraphics[width=0.3\textwidth]{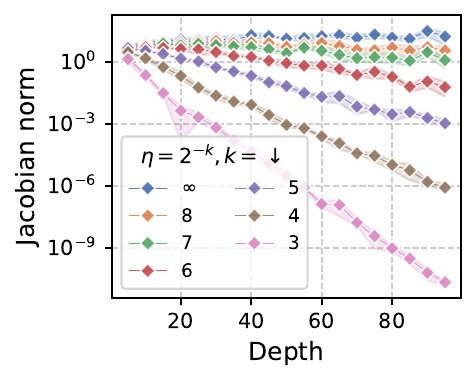}
            \includegraphics[width=0.3\textwidth]{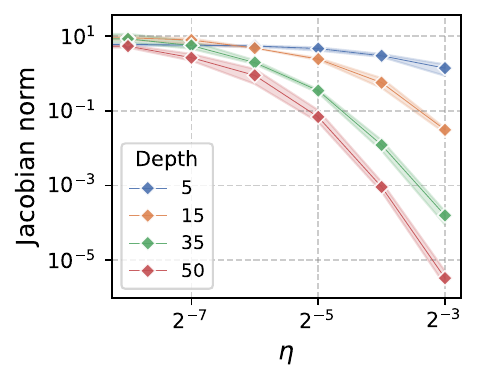}
        \end{center}
        \caption{\small{Illustration of the Jacobian norm for a randomly selected input with an MLP architecture of width $n = 256$ and varying depths. All results are averaged over 3 runs, and confidence intervals are highlighted with shaded areas. \textbf{(Left)} Impact of Depth on the Jacobian norm for different correlation levels. \textbf{(Right)} Impact of the injection of the correlation between the weights on the Jacobian norm.}}
        \label{fig:jacnorm_init_noni.i.d}
    \end{figure}

    It is reasonable to anticipate that when the correlation between weights is small enough, the Jacobian norm will display behavior similar to that observed in the i.i.d case. However, quantifying the level of correlation required for such similarity to hold is not straightforward.
    \cref{thm:criticality_non_i.i.d} provides a width-dependent  bound on the correlation. Provided that the correlation is significantly smaller than $n^{-1}$ up to a logarithmic term, the Jacobian norm has the same fixed-depth large-width limit as the iid case. In \cref{fig:jacnorm_init_noni.i.d}, we demonstrate the impact of the correlation between weights on the Jacobian norm as depth grows. In our simulations, the weights are generated as $W_k^{ij} = W_{k, ind}^{ij} + \eta w^k$, where $w^k, W_{k, ind}^{ij} \sim \mathcal{N}(0, 2/n)$, and~$\eta$ is held constant. In this setup, we have $\corr(W_k^{ij}, W_k^{ml}) = \frac{\eta^2}{1 + \eta^2}\approx \eta^2$ when~$\eta$ is small. For $n=256 = 2^8$, the condition specified in \cref{thm:criticality_non_i.i.d} translates to $\eta \ll 2^{-4}$. As seen in the figure, for $\eta \in \{2^{-8}, 2^{-7}\}$, the Jacobian norm closely matches the i.i.d case (represented by the blue curve), particularly when $L \leq 40$. As depth increases, one would expect that the difference between the correlated and the i.i.d Jacobian norms to become more pronounced. This is due to the fact that, given a fixed depth, the result is valid in the infinite-width limit\footnote{Note that Theorem~\ref{thm:criticality_i.i.d} and the results in this paper are fixed-depth results and do not address the joint limit where depth and width diverge simultaneously.}.

    \section{Further empirical results}\label{app:further_empirical_results}

    In this section, we conduct several experiments to verify some theoretical results. All the networks are trained with SGD (Stochastic Gradient Descent), and learning rate is tuned using a grid search in the set $\{1\text{e-}1, 1\text{e-}2, 1\text{e-}3, 1\text{e-}4\}$.
    \subsection{Empirical verification of \cref{approx:i.i.d_bernoulli}}
    We conduct a statistical independence test and additional visualization methods in order to verify this approximation.\\
    \begin{figure}
        \centering
        \includegraphics[width=0.5\linewidth]{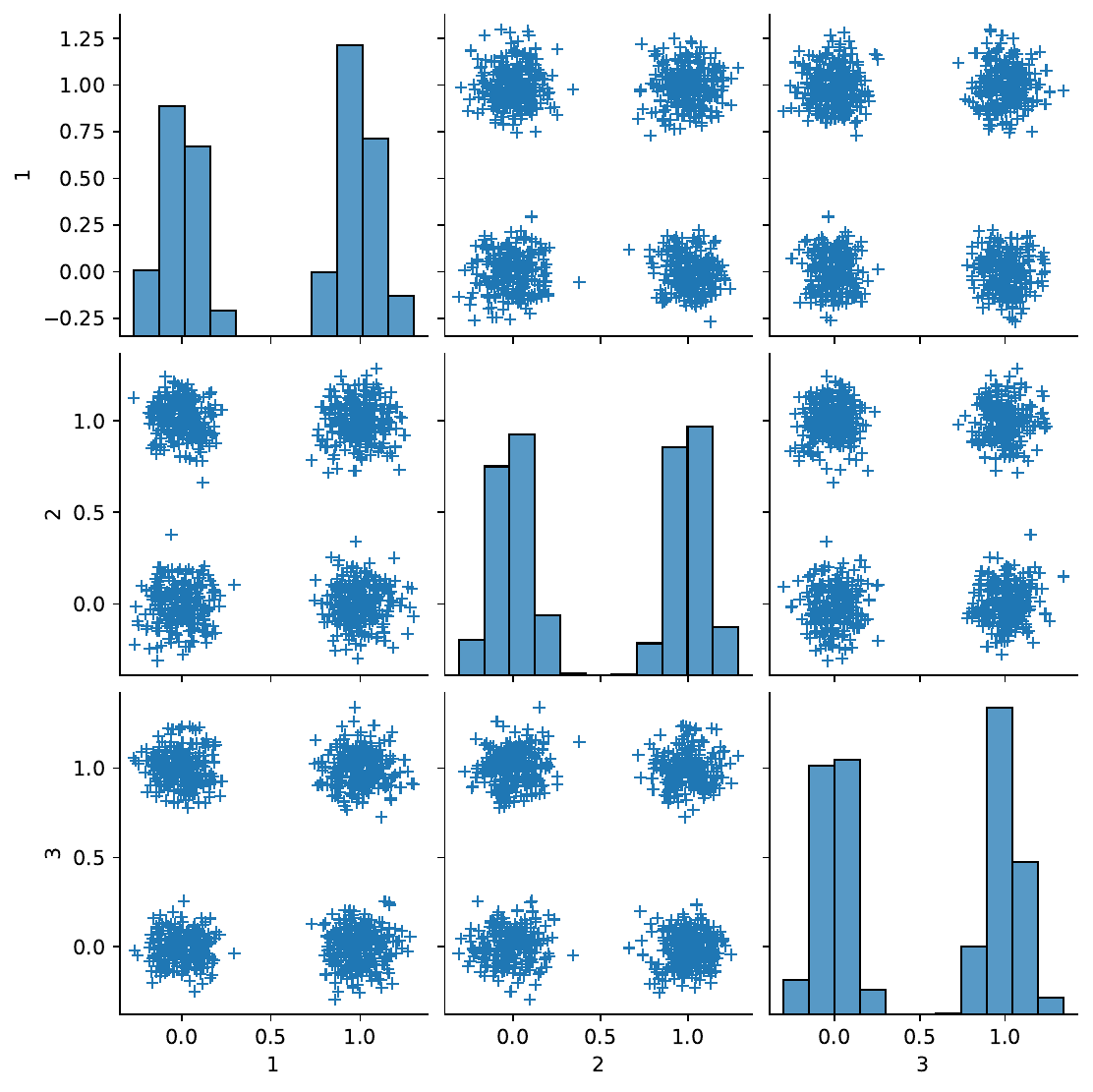}
        \caption{Joint distributions of three randomly selected entries of~$D_l$ (denoted by~$1$, $2$, and~$3$) for $l=10$ in a depth $L=30$ and width $n=100$ MLP with a randomly selected input, based on $N=1000$ simulations. Since the values of the entries are binary ($0$ or~$1$) we added random Gaussian noise (variance~$0.01$) to the points for better visibility. }
        \label{fig:ind_iid_ber}
    \end{figure}
    \emph{Diagonal entries of~$D_l$:} In \cref{fig:ind_iid_ber}, we show the joint distributions of three randomly selected entries of~$D_l$ for $l=10$ in a depth $L=30$ and width $n=100$ (MLP network with ReLU activation function). Since the values of the entries are binary ($0$ or~$1$) we added random Gaussian noise (variance~$0.01$) to the points for better visibility. The two main observations are the absence of correlation between the values of the entries, and that each entry has approximately probability~$1/2$ of being equal to~$1$, thus confirming the validity of \cref{approx:i.i.d_bernoulli} in the iid weights case. \\

    \textit{Chi-squared independence test. } We further run a chi-squared independence test between two randomly selected entries of~$D_l$ and we obtained the following results: $\chi^2(1) = 0.1379$, and $\textup{p-value} = 0.73$, and thus the ``$H_0$'' hypothesis (independent random variables) cannot be rejected. Another observation is that the p-value seems to have a uniform distribution as we change the random seed, which further supports the independence hypothesis.\footnote{It is well known that the distribution of the p-value under the~$H_0$ hypothesis is uniform in~$[0,1]$}

    \textit{Independence of~$W_l$ and~$D_l$.} There is no standard method to evaluate the independence between a discrete and a continuous random variable. We use the following heuristic to evaluate dependence between the matrices~$W_l$ and~$B_l$: we compute the statistics $T_W = \frac{1}{n^2} \sum_{1\leq i,j \leq n} 1_{W_l^{ij} > 0}$ and $T_D = \frac{1}{n} \sum_{1\leq i \leq n} 1_{D_l^{i} > 0}$ and study the correlation between them.

    \begin{figure}[h]
        \centering
        \includegraphics[width=0.5\linewidth]{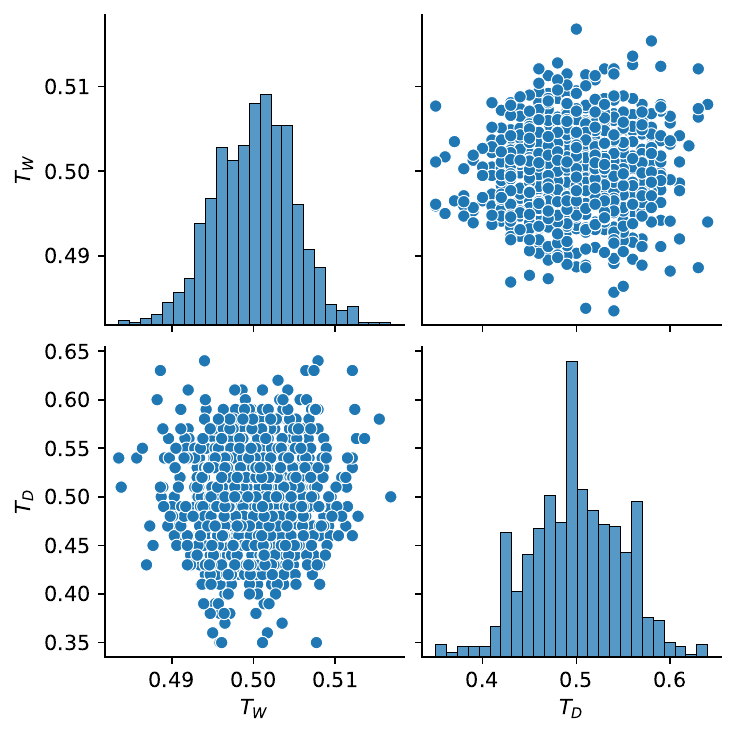}
        \caption{Joint distribution of $(T_W, T_D)$ for~$D_l$ and~$W_l$ ($l=10$) in a $L=30$ and $n=100$ MLP with a randomly selected input, based on $N=1000$ simulations.}
        \label{fig:ind_iid_WvsD}
    \end{figure}
    \cref{fig:ind_iid_WvsD} shows the joint distribution of $(T_W, T_D)$ for~$D_l$ and~$W_l$ ($l=10$) in a $L=30$ and $n=100$ MLP. No clear correlation can be observed from the histograms. This supports the approximation of independence between~$W$ and~$D$. In the case of dependent weights, \cref{fig:ind_dependent_ber} shows the joint distributions of three randomly selected diagonal entries of $D$ with the only difference being that the weights are now correlated. With weak correlation (left figure), it appears that \cref{approx:i.i.d_bernoulli} is still a good approximation. However, with high correlation (right figure), it is clear that the diagonal entries are not dependent and the approximation is not valid in this case. The statistic become correlated when we increase the correlation level, confirming the validity of this simple heuristic.

\begin{figure}[h]
    \centering
    \includegraphics[width=0.45\linewidth]{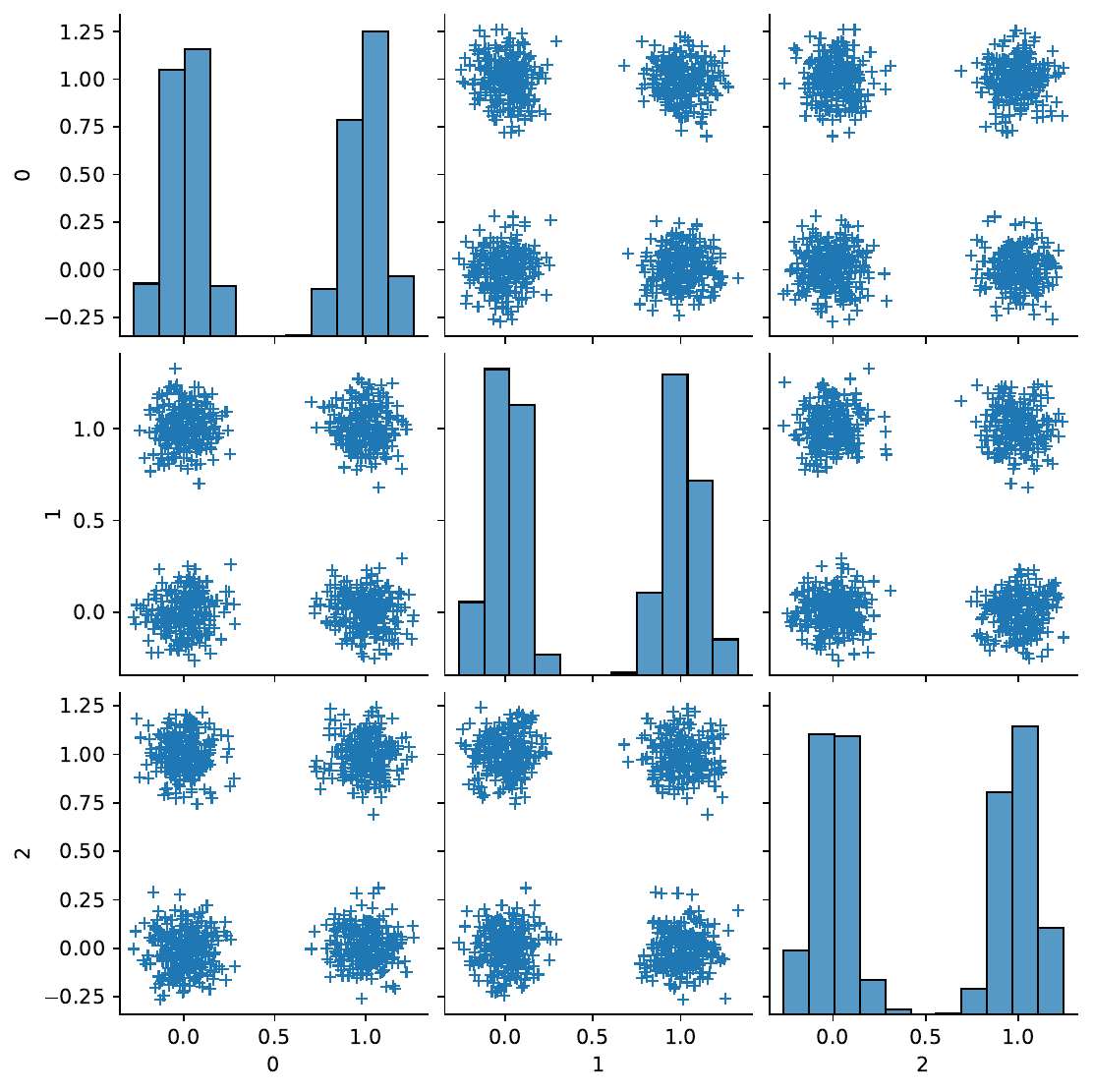}
\includegraphics[width=0.45\linewidth]{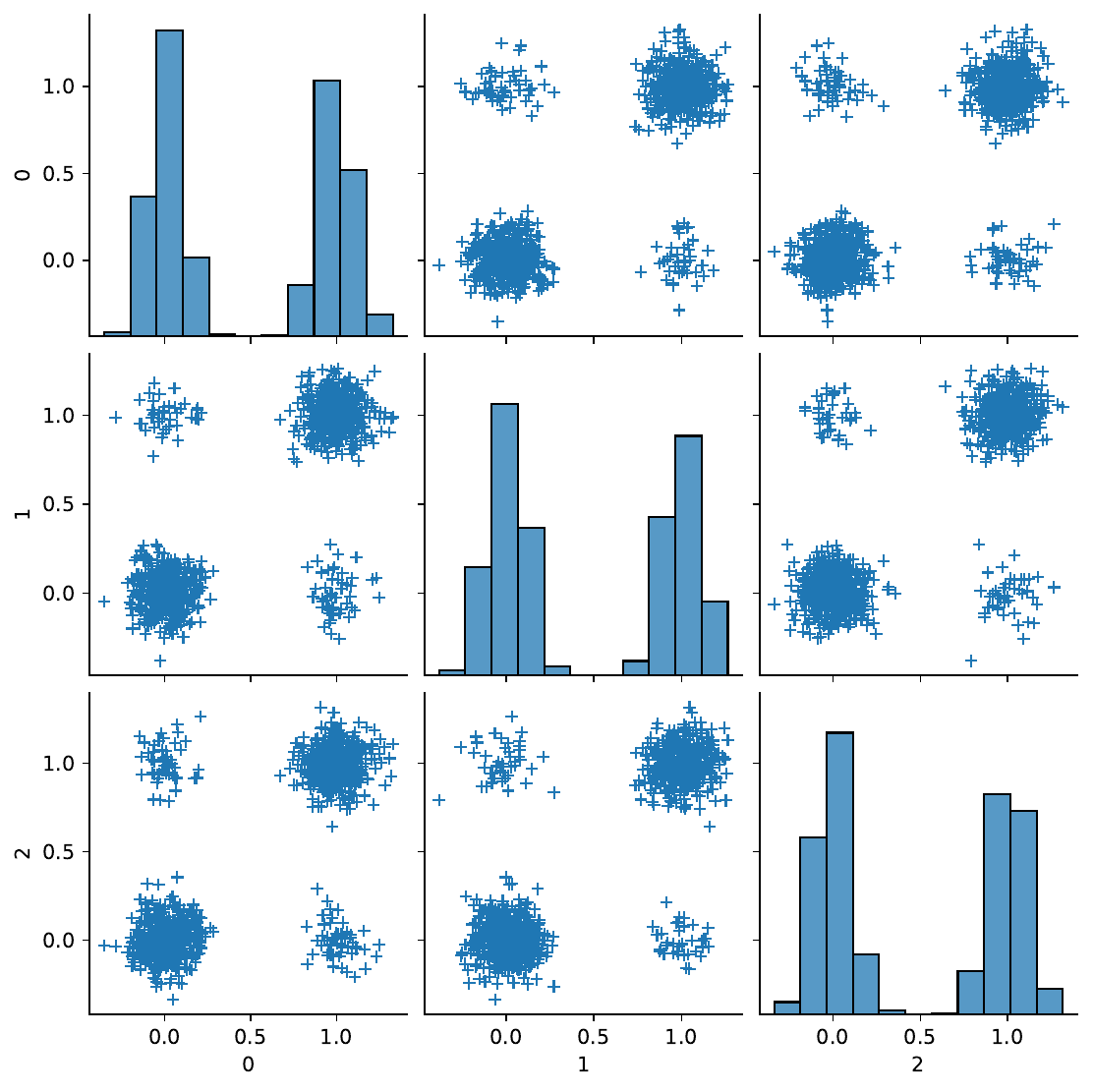}
    \caption{Same as \cref{fig:ind_iid_ber} with dependent weights given by $\Tilde{W}^{ij}_l = W^{ij}_l + \eta w$ as in \cref{fig:jacnorm_init_noni.i.d}. \textbf{(Left)} $\eta = 0.1$. \textbf{(Right)} $\eta = 0.9$. }
    \label{fig:ind_dependent_ber}
\end{figure}
\paragraph{Diagonal entries of $D_l$.}

    \subsection{Impact of scaling in pruned networks}
    In \cref{fig:val_accuracy_pruned_random} and \cref{fig:val_loss_pruned_random}, we report the test accuracy/loss of an MLP of width $n=256$ of varying depths trained on Fashion-MNIST. Incorporating the scaling factor significantly improves the trainability of the network after pruning. The edge of stability can also be observed in terms of trainability as predicted in \cref{thm:fixed_sparsity}.
    \begin{figure}[h]
        \centering
        \includegraphics[width=.3\textwidth]{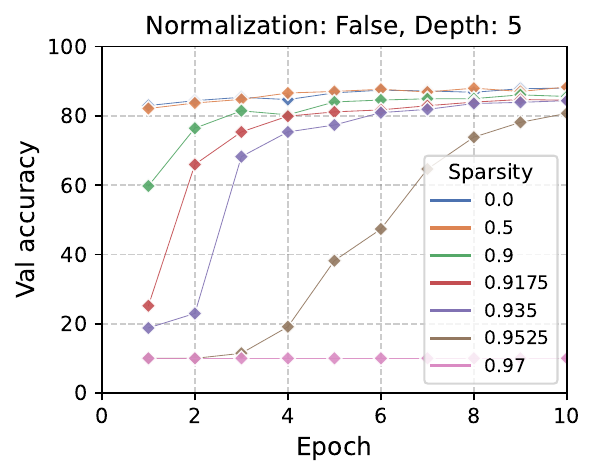}
        \includegraphics[width=.3\textwidth]{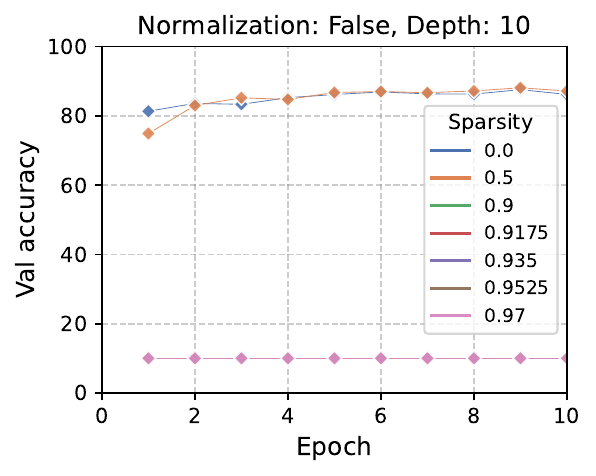}
        \includegraphics[width=.3\textwidth]{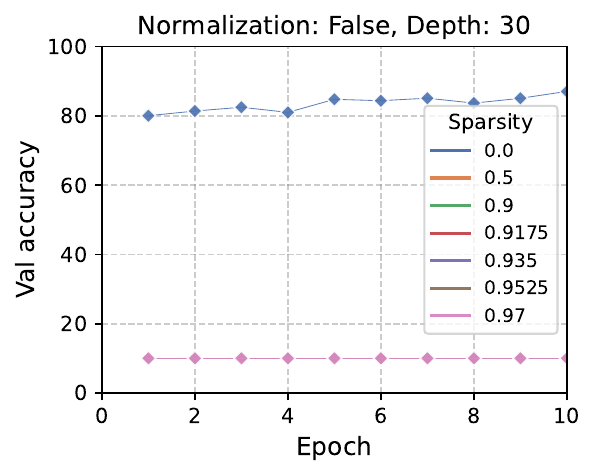}
        \includegraphics[width=.3\textwidth]{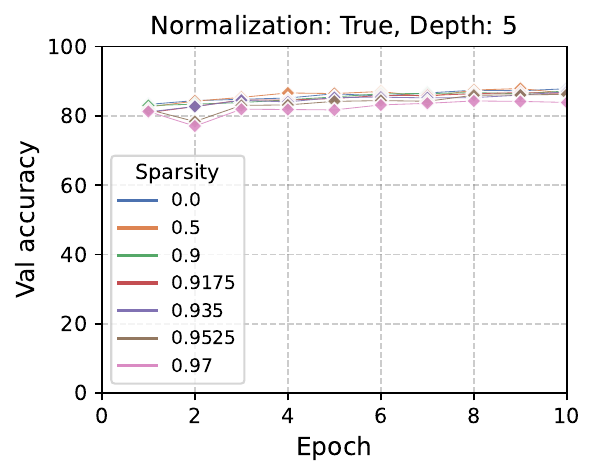}
        \includegraphics[width=.3\textwidth]{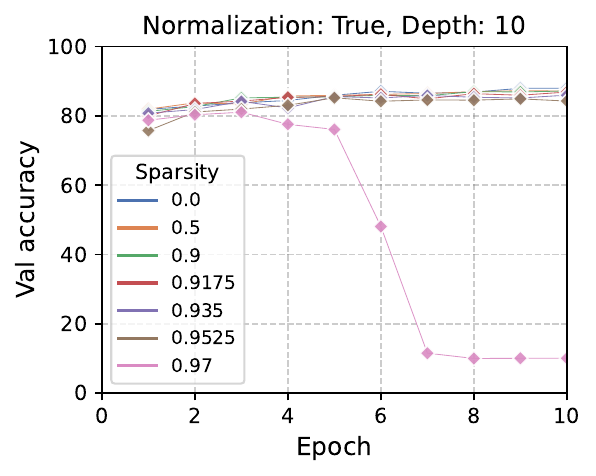}
        \includegraphics[width=.3\textwidth]{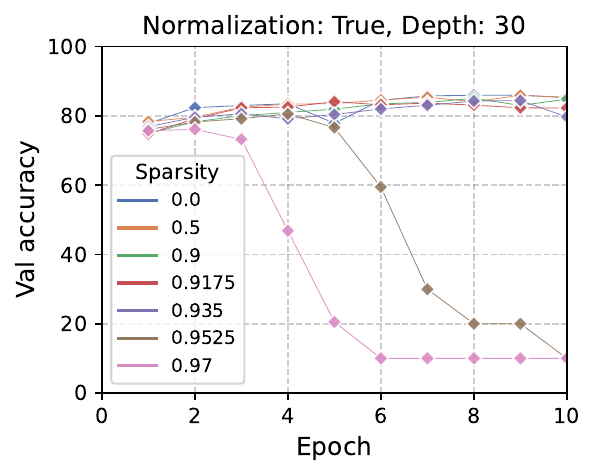}
        \caption{Test (Validation) accuracy of a width $n=256$ \emph{randomly pruned} MLP of varying depths trained on Fashion-MNIST with learning rate $\gamma = 0.01$. Top figures show the results without scaling (we refer to this by normalization) while the bottom figures show the results with scaling.}
        \label{fig:val_accuracy_pruned_random}
    \end{figure}

    \begin{figure}[h]
        \centering
        \includegraphics[width=.3\textwidth]{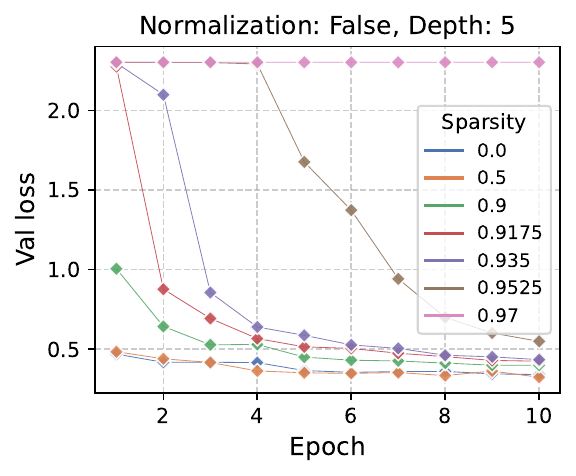}
        \includegraphics[width=.3\textwidth]{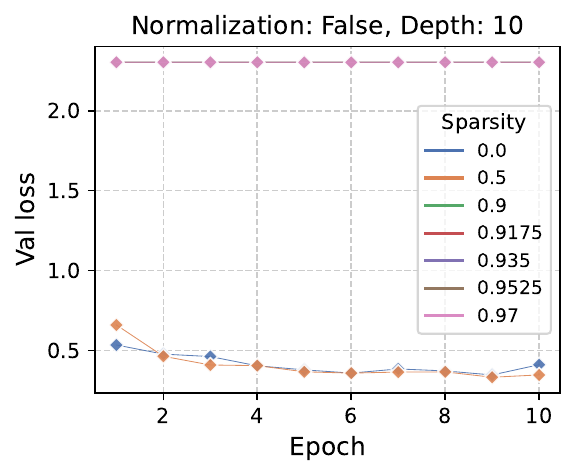}
        \includegraphics[width=.3\textwidth]{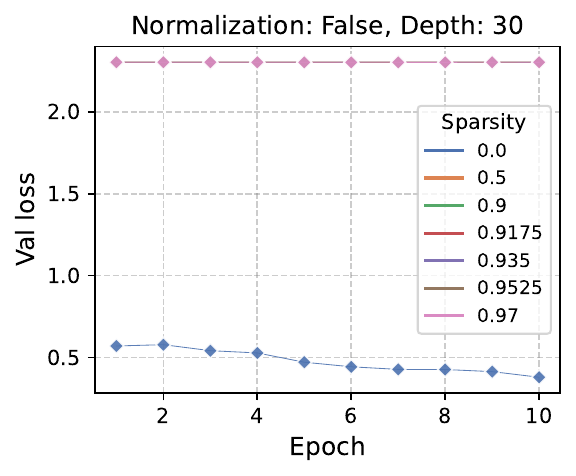}
        \includegraphics[width=.3\textwidth]{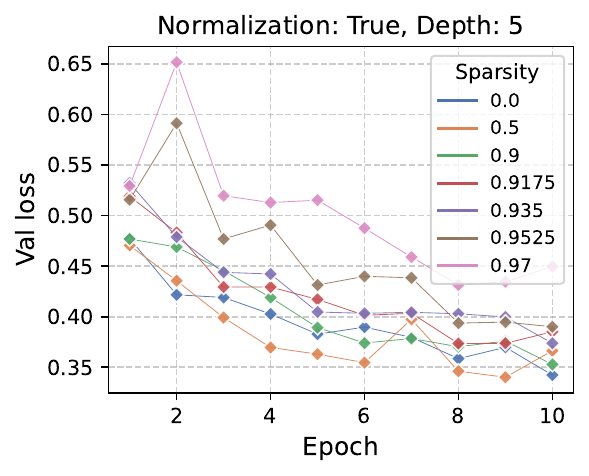}
        \includegraphics[width=.3\textwidth]{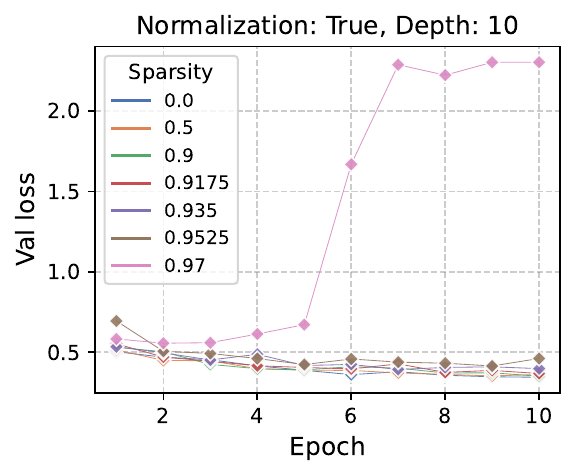}
        \includegraphics[width=.3\textwidth]{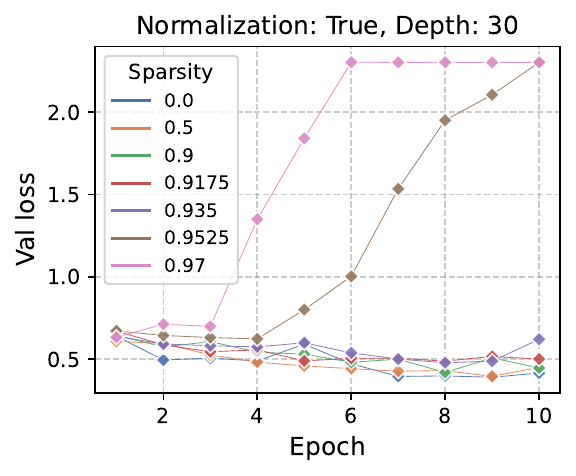}
        \caption{Test loss of a width $n=256$ \emph{randomly pruned} MLP of varying depths trained on Fashion-MNIST with learning rate $\gamma=0.01$. Top figures show the results without normalization while the bottom figures show the results with normalization.}
        \label{fig:val_loss_pruned_random}
    \end{figure}

    \subsection{Impact of correlated weights}
    In \cref{fig:test_error_iid_dependent}, we depict the test error after convergence in two settings: IID initialized MLP trained on MNIST, and Dependent Weights Initialized MLP trained on Fashion-MNIST. The results support our theoretical findings.
    \begin{figure}[h]
        \centering
        \includegraphics[width=0.3\linewidth]{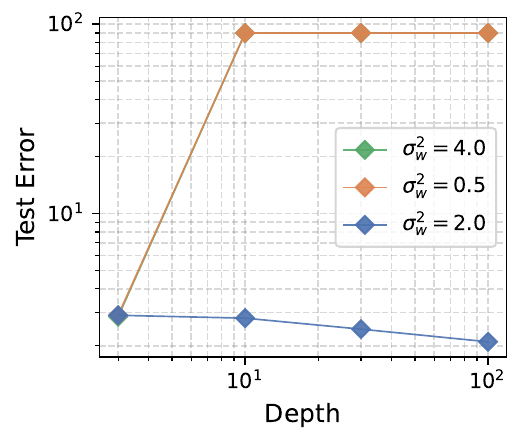}
        \includegraphics[width=0.32\linewidth]{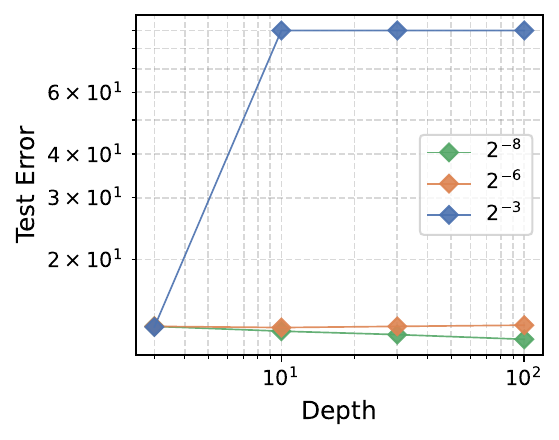}
        \caption{\textbf{(Left)} Test error after convergence ($70$ epochs) as a function of depth in an MLP with width $n=256$ trained with SGD with learning rate $\gamma = 0.001$ on MNIST. The results are shown for 3 different choices of~$\sigma_w$, the variance of the iid weights. The critical initialization given by $\sigma_w^2=2$ guarantees trainability up to depth 100. Note that the orange and green curves coincide, hence the green curve is not visible. \textbf{(Right)} Test error after convergence ($70$ epochs) as a function of depth in an MLP with width $n=256$ trained with SGD with learning rate $\gamma = 0.001$ on Fashion-MNIST. The results are shown for 3 different choices of the correlation parameter~$\eta$. When the correlation is low, the network remains trainable for large depths, which is not the case for high correlation levels.}
        \label{fig:test_error_iid_dependent}
    \end{figure}

    \section{Conclusion and Limitations}

    In this paper, we provided an analysis of the Jacobian norm in DNNs in two different contexts: 1) dependent weights, and 2) sparse networks. Our findings show that depth stability can be guaranteed in both cases under some conditions: 1) for dependent weights, correlation should be bounded by $o(n^{-3/2}/\log(n))$, and 2) for sparse networks, weights must be re-scaled after pruning. In this regard, our work expands the existing literature on depth stability and justifies the crucial role of initialization and weight scaling, often observed in practice. 
    
    However, one limitations of our theory is that it currently only applies to the MLP architecture. Besides, for sparse networks, we only consider random and magnitude-based pruning. Extending these results to more modern architectures and pruning algorithms is an interesting question for future work.

    \newcommand{\etalchar}[1]{$^{#1}$}
    \providecommand{\bysame}{\leavevmode\hbox to3em{\hrulefill}\thinspace}
    \renewcommand{\MR}[1]{%
        \href{https://www.ams.org/mathscinet-getitem?mr=#1}{MR\,#1}
    }
    \newcommand{\arXiv}[1]{%
        \href{https://arxiv.org/abs/#1}{arXiv:#1}
    }

\end{document}